\documentclass{article}

\usepackage{microtype}
\usepackage{graphicx}
\usepackage{booktabs} %

\usepackage{hyperref}

\usepackage{amsmath,amsthm,amssymb}
\usepackage{subcaption}
\usepackage{multirow}
\usepackage{verbatim}
\DeclareMathOperator*{\argmax}{arg\,max}

\newtheorem{theorem}{Theorem}[section]

\newtheorem{lemma}[theorem]{Lemma}

\usepackage{array}
\newcolumntype{L}[1]{>{\raggedright\let\newline\\\arraybackslash\hspace{0pt}}m{#1}}
\newcolumntype{C}[1]{>{\centering\let\\}m{#1}}

\newcolumntype{R}[1]{>{\raggedleft\let\newline\\\arraybackslash\hspace{0pt}}m{#1}}

\usepackage[accepted]{icml2019}
\usepackage{breqn}

\icmltitlerunning{Anomaly Detection With Multiple-Hypotheses Predictions}

\begin{document}

\twocolumn[
\icmltitle{Anomaly Detection With Multiple-Hypotheses Predictions}

\icmlsetsymbol{equal}{*}

\begin{icmlauthorlist}
\icmlauthor{Duc Tam Nguyen}{Freiburg,Bo}
\icmlauthor{Zhongyu Lou}{Bo}
\icmlauthor{Michael Klar}{Bo}
\icmlauthor{Thomas Brox}{Freiburg}
\end{icmlauthorlist}

\icmlaffiliation{Bo}{Corporate Research, Robert Bosch GmbH, Renningen, Germany}
\icmlaffiliation{Freiburg}{Computer Vision Group, University of Freiburg, Freiburg, Germany }

\icmlcorrespondingauthor{Duc Tam Nguyen}{Nguyen@informatik.uni-freiburg.de}
\icmlcorrespondingauthor{Zhongyu Lou}{Zhongyu.Lou@de.bosch.com}
\icmlcorrespondingauthor{Michael Klar}{Michael.Klar2@de.bosch.com}
\icmlcorrespondingauthor{Thomas Brox}{Brox@informatik.uni-freiburg.de}

\icmlkeywords{Machine Learning, ICML}

\vskip 0.3in
]

\printAffiliationsAndNotice{}  %

\begin{abstract}
In one-class-learning tasks, only the normal case (foreground) can be modeled with data, whereas the variation of all possible anomalies is too erratic to be described by samples. 
Thus, due to the lack of representative data, the wide-spread discriminative approaches cannot cover such learning tasks, and rather generative models,
which attempt to learn the input density of the foreground, are used. 
However, generative models suffer from a large input dimensionality (as in images) and are typically inefficient learners.
We propose to learn the data distribution of the foreground more efficiently with a \emph{multi-hypotheses autoencoder}. Moreover, the model is criticized by a \emph{discriminator}, which prevents artificial data modes not supported by data, and enforces diversity across hypotheses. 
Our multiple-hypotheses-based anomaly detection framework allows the reliable identification of out-of-distribution samples. For anomaly detection on CIFAR-10, it yields up to 3.9\% points improvement over previously reported results. On a real anomaly detection task, the approach reduces the error of the baseline models from 6.8\% to 1.5\%.
\end{abstract}
\section{Introduction}
Anomaly detection classifies a sample as normal or abnormal. In many applications, however, it must be treated as a one-class-learning problem, since the abnormal class cannot be defined sufficiently by samples. Samples of the abnormal class can be extremely rare, or they do not cover the full space of possible anomalies. For instance, in an autonomous driving system, we may have a test case with a bear or a kangaroo on the road. For defect detection in manufacturing, new, unknown production anomalies due to critical changes in the production environment can appear. In medical data analysis, there can be unknown deviations from the healthy state. In all these cases, the well-studied discriminative models, where decision boundaries of classifiers are learned from training samples of all classes, cannot be applied. The decision boundary learning of discriminative models will be dominated by the normal class, which will negatively influence the classification performance.

Anomaly detection as one-class learning is typically approached by generative, reconstruction-based methods \citep{zong2018deep}. 
They approximate the input distribution of the normal cases by parametric models, which allow them to reconstruct input samples from this distribution. 
At test time, the data negative log-likelihood serves as an anomaly-score. 
In the case of high-dimensional inputs, such as images, learning a representative distribution model of the normal class is hard and requires many samples.

Autoencoder-based approaches, such as the variational autoencoder \citep{rezende2014stochastic,kingma2013auto}, mitigate the problem by learning a mapping to a lower-dimensional representation, where the actual distribution is modeled. In principle, the nonlinear mappings in the encoder and decoder allow the model to cover multi-modal distributions in the input space.  However, in practice, autoencoders tend to yield blurry reconstructions, since they regress mostly the conditional mean rather than the actual multi-modal distribution (see Fig. \ref{fig:Teaser:Metal_anomaly} for an example on a metal anomaly dataset). 
Due to multiple modes in the actual distribution, the approximation with the mean predicts high probabilities in areas not supported by samples. The blurry reconstructions in Fig. \ref{fig:Teaser:Metal_anomaly} should have a low probability and be classified as anomalies, but instead they have the highest likelihood under the learned autoencoder. This is fatal for anomaly detection. 

\begin{figure*}
\centering
        \begin{subfigure}{.15\linewidth}
            \centering
      \includegraphics[width=.99\linewidth]{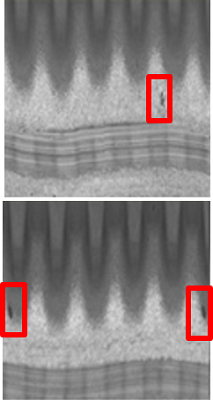}
      \caption{Test images}
      \label{fig:teaser-sfig1}
    \end{subfigure}%
    \begin{subfigure}{.32\textwidth}
      \centering
      \includegraphics[width=.95\linewidth]{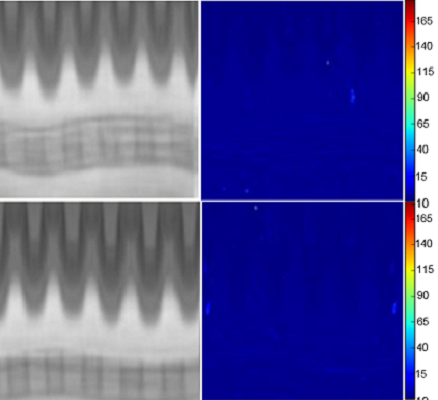}
      \caption{Autonencoder reconstructions}
      \label{fig:teaser-sfig2}
    \end{subfigure}
    \begin{subfigure}{.32\textwidth}
      \includegraphics[width=.95\linewidth]{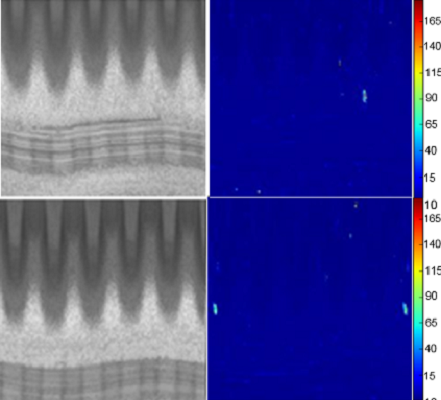}
      \caption{ConAD reconstructions}
      \label{fig:sfig3}
    \end{subfigure}
    \caption{Detection of anomalies on a Metal Anomaly dataset. (a) Test images showing anomalies (black spots). (b) An Autoencoder-based approach produces blurry reconstructions to express model uncertainty. The blurriness falsifies reconstruction errors (and hence anomaly scores)(c) Our model: Consistency-based anomaly detection (ConAD) gives the network more expressive power with a multi-headed decoder (also known as multiple-hypotheses networks). The resulting anomaly scores are hence much clearer in our framework ConAD. }
    \label{fig:Teaser:Metal_anomaly}
\end{figure*}

\begin{figure*}[h]
\centering
     \begin{subfigure}{.19\textwidth}
      \includegraphics[width=1.05\linewidth]{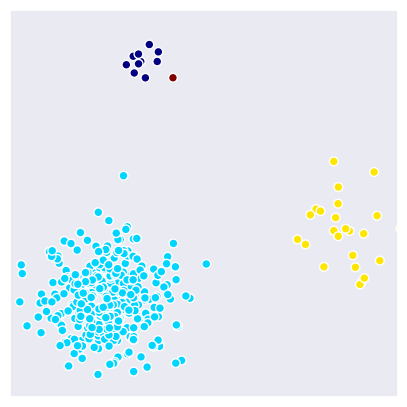}
      \caption{Cond. space}
      \label{fig:sfig1LOF}
    \end{subfigure}%
    \begin{subfigure}{.19\textwidth}
      \includegraphics[width=1.05\linewidth]{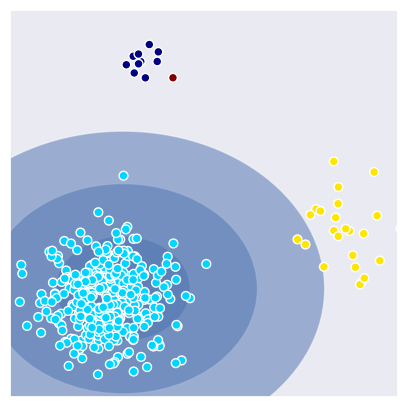}
      \caption{Autoencoder}
      \label{fig:sfig2LOF}
    \end{subfigure}
    \begin{subfigure}{.19\textwidth}
      \includegraphics[width=1.05\linewidth]{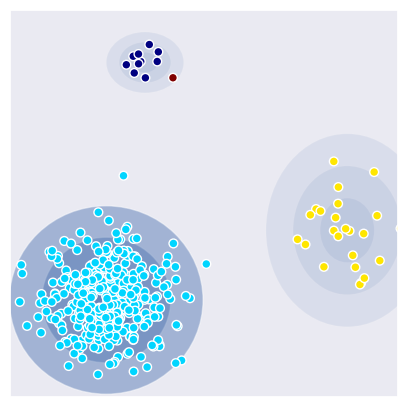}
      \caption{MDN}
      \label{fig:sfig3LOF}
    \end{subfigure}
    \begin{subfigure}{.19\textwidth}
      \includegraphics[width=1.05\linewidth]{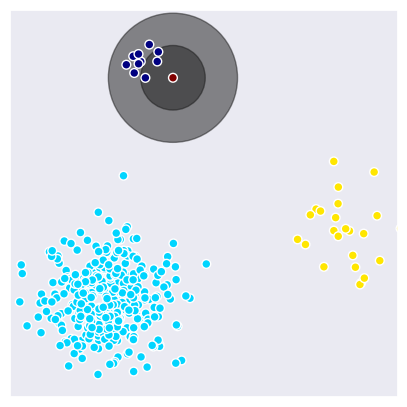}
      \caption{LOF}
      \label{fig:sfig4LOF}
    \end{subfigure}
        \begin{subfigure}{.19\textwidth}
      \includegraphics[width=1.05\linewidth]{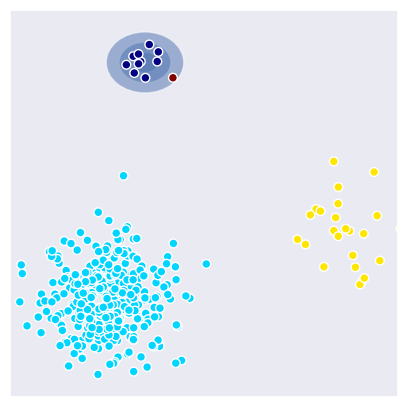}
       \caption{Our model}
      \label{fig:sfig5LOF}
    \end{subfigure}
    \caption{Illustration of the different anomaly detection strategies. (a) In this example, two dimensions with details that are hard to capture in the conditional space are shown. The red dot is a new point.  Dark blue indicates high likelihood, black indicates the neighborhood considered. 
    The autoencoder  (b) cannot deal with the multi-modal distribution. The mixture density network (c) in principle can do so, but recognition of the sample as a normal case is very brittle and will fail in case of mode collapse. Local-Outlier-Factor (d) makes a decision based on the data samples closest to the input sample. Our model (e) learns multiple local distributions and uses the data likelihood of the closest one as the anomaly score. } 
    \label{fig:MDN-LOF-MHP}
\end{figure*}

Alternatively, mixture density networks  \citep{bishop1994mixture} learn a conditional Gaussian \emph{mixture distribution}. They directly estimate local densities that are coupled to a global density estimate via mixing coefficients. Anomaly scores for new points can be estimated using the data likelihood (see Appendix).
However, global, multi-modal distribution estimation is a hard learning problem with many problems in practice. In particular, mixture density networks tend to suffer from mode collapse in high-dimensional data spaces, i.e., the relevant data modes needed to distinguish rare but normal data from anomalies will be missed. 

Simple nearest neighbor analysis, such as the Local-outlier-factor \cite{breunig2000lof}, operates in image-space directly without training. While this is a simple and sometimes effective baseline, such local analysis is inefficient in very high-dimensional spaces and is slow at test time.
Fig. \ref{fig:MDN-LOF-MHP} illustrates these different strategies in a simple, two-dimensional example. 

In this work, we propose the use of multiple-hypotheses networks \citep{rupprecht_learning_2016,koltun_multi_choise, ilg_uncertainty_2018, bhattacharyya2018accurate} for anomaly detection to provide a more fine-grained description of the data distribution than with a single-headed network. 
In conjunction with a variational autoencoder, the multiple hypotheses can be realized with a multi-headed decoder. Concretely, \emph{each network head} may predict a Gaussian density estimate.
Hypotheses form clusters in the data space and can capture model uncertainty not encoded by the latent code. 

Multiple-hypotheses networks have not yet been applied to anomaly detection due to several difficulties in training these networks to produce a multi-modal distribution consistent with the training distribution. 
The loosely coupled hypotheses branches are typically learned with a winner-takes-all loss, where all learning signal is transferred to one single best branch. 
Hence, bad hypotheses branches are not penalized and may support non-existing data regions. These artificial data modes cannot be distinguished from normal data.
This is an undesired property for anomaly detection and becomes more severe with an increasing number of hypotheses.

We mitigate the problem of artificial data modes by combining multiple-hypotheses learning with a discriminator D as a critic. The discriminator ensures the consistency of estimated data modes with the real data distribution. Fig. \ref{fig:training and testing}  shows the scheme of the framework.

This approach combines ideas from all three previous paradigms: the latent code of a variational autoencoder yields a way to efficiently realize a generative model that can act in a rather low-dimensional space; the multiple hypotheses are related to the mixture density of mixture density networks, yet without the global component, which leads to mode collapse. 

We evaluate the anomaly detection performance of our approach on CIFAR-10 and a real anomaly image dataset, the \emph{Metal Anomaly dataset} with images showing a structured metal surface, where anomalies in the form of scratches, dents or texture differences are to be detected.
We show that anomaly detection performance with multiple-hypotheses networks is significantly better compared to single-hypotheses networks. On CIFAR-10, our proposed ConAD framework (consistency-based anomaly detection) improves on previously published results. 
Furthermore, we show a large performance gap between ConAD and mixture density networks. 
This indicates that anomaly score estimation based on the global neighborhood (or data likelihood) is inferior to local neighborhood consideration.

\section{Anomaly detection with multi-hypotheses variational autoencoders}%

\subsection{Training and testing for anomaly detection}

\begin{figure}[t]
\centering
    \begin{subfigure}{.43\textwidth}
      \includegraphics[width=0.99\linewidth]{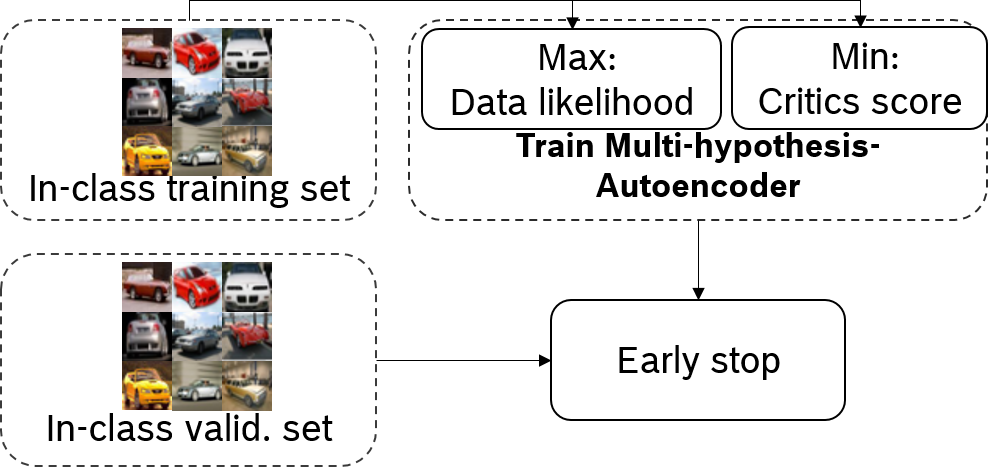}
      \caption{}
      \label{fig:training}
      \end{subfigure}

    \begin{subfigure}{.43\textwidth}
      \includegraphics[width=0.99\linewidth]{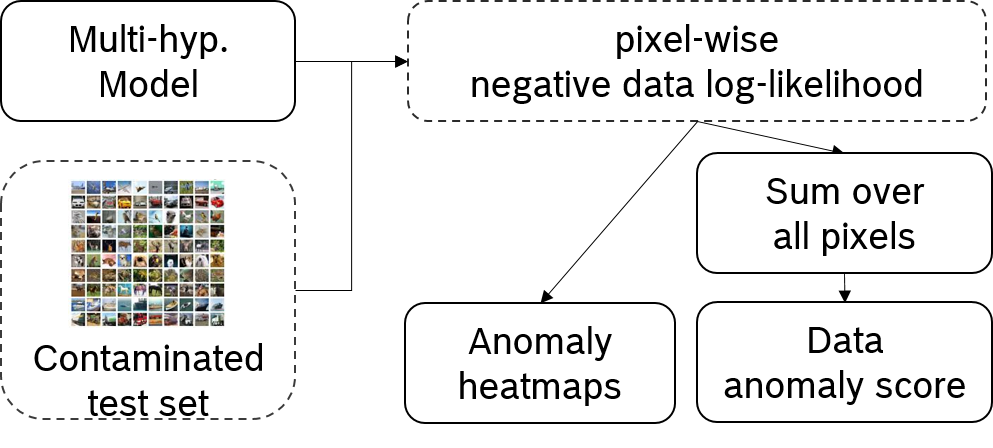}
      \caption{ }          
      \label{fig:testing}
    \end{subfigure}
    \caption{Training and testing overview of the proposed anomaly detection framework. (a) shows training the model to capture the normal data distribution. For the distribution learning, we use a multiple-hypotheses variational autoencoder (Fig. \ref{fig:Multiple-hyp-VAE}) with discriminator training (Fig. \ref{fig:discriminator training}). During training, only data from the normal case are used. 
    (b) At test time, the data likelihood is used for detecting anomalies. A low likelihood indicates an out-of-distribution sample, i.e., an anomaly. }
    \label{fig:training and testing}
\end{figure}

    Fig. \ref{fig:training and testing} shows the training and testing within our framework. The multiple-hypothesis variational autoencoder (Fig. \ref{fig:Multiple-hyp-VAE}) uses the data from the normal case for distribution learning. The learning is performed with the maximum likelihood and critics minimizing objectives (Fig. \ref{fig:discriminator training}). 

    At test time (Fig \ref{fig:testing}), the test set is contaminated with samples from other classes (anomalies). For each sample, the data negative log-likelihood under the learned multi-hypothesis model is used as an anomaly score. The discriminator only acts as a critic during training and is not required at test time.
\subsection{Multiple-hypotheses variational autoencoder}
\label{section: model archiitecture}
For fine-grained data description, we learn a distribution with a multiple-hypotheses autoencoder. 
Figure \ref{fig:Multiple-hyp-VAE} shows our multiple-hypotheses variational autoencoder. The last layer (head) of the decoder is split into $H$ branches to provide $H$ different hypotheses. The outputs of each branch are the parameters of an independent Gaussian for each pixel.

\begin{figure}[t]
\centering
    \includegraphics[width=0.80\linewidth]{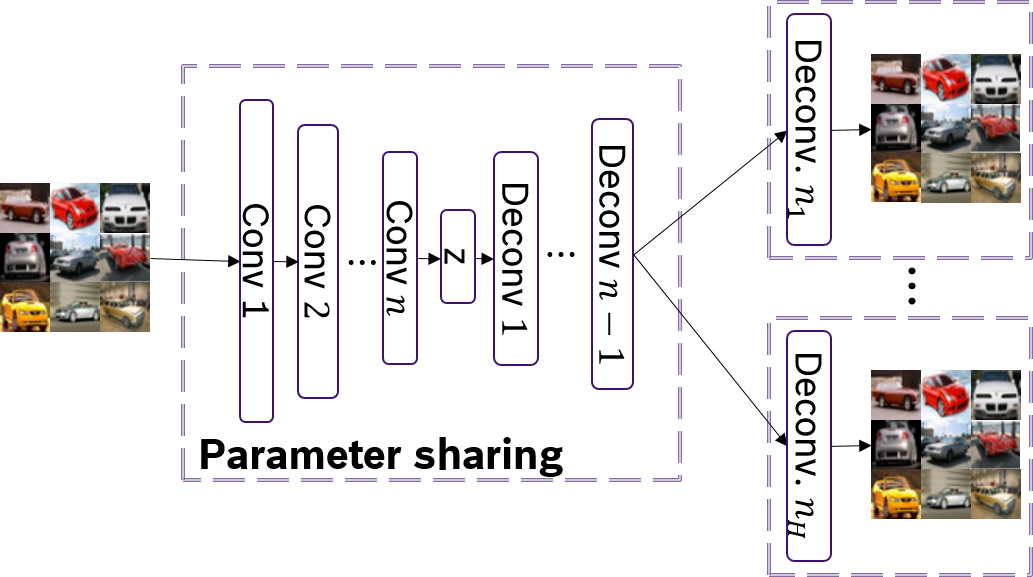}
    \caption{Multi-headed variational autoencoder. All heads share the same encoder, the same latent code, and large parts of the decoder, but the last layers create different hypotheses.}
    \label{fig:Multiple-hyp-VAE}
\end{figure}

    In the basic training procedure without discriminator training, the multiple-hypotheses autoencoder is trained with the winner-takes-all (WTA) loss:

        \begin{dmath}
        L_{WTA}(x_i|\theta_h) =  E_{z_k\sim q_\phi(z|x)}\left[\log p_{\theta_h}(x_i|z_k)\right]  
        \\ \textbf{ s.t. }  h = \argmax_j E_{z_k\sim q_\phi(z|x)}\left[\log p_{\theta_j}(x_i|z_k)\right],  
        \label{eq:WTA-simple}
        \end{dmath}
        
whereby $\theta_j$  is the parameter set of hypothesis branch $j$, $\theta_h$ the best hypothesis w.r.t. the data likelihood of sample $x_i$, $z_k$ is the noise and $q_\phi$ the distribution after the encoder. Only the network head with the best-matching hypothesis concerning the training sample receives the learning signal. %

\subsection{Training with discriminator as a critic}
\label{sec:discrim training}
\begin{figure}[t]
\centering
      \includegraphics[width=0.99\linewidth]{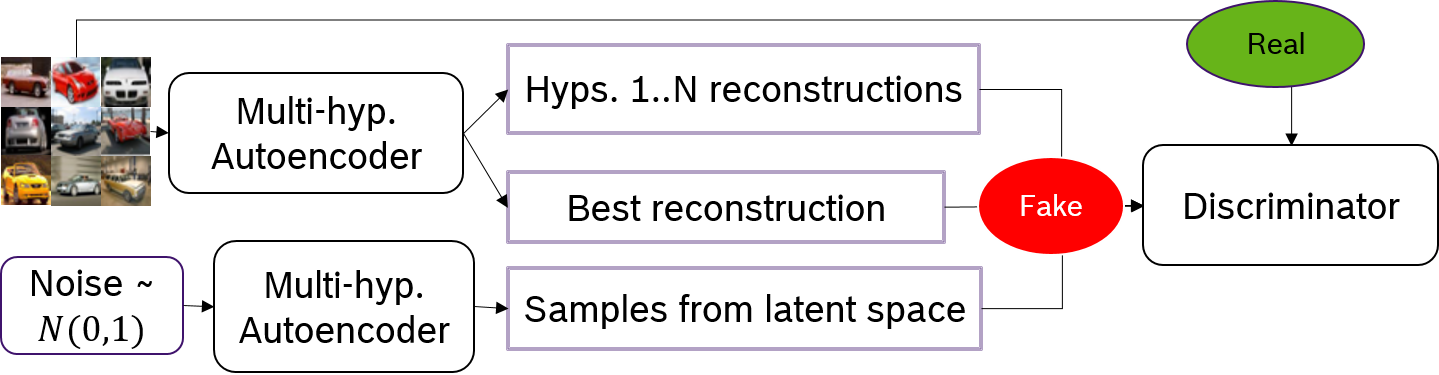}
    \caption{Discriminator training in the context of the multiple-hypotheses autoencoder. As in usual discriminator training, an image from the training set and a randomly sampled image are labeled as real and fake respectively. Additional fake samples are generated by the autoencoder.} 
    \label{fig:discriminator training}
\end{figure}

When learning with the winner-takes-all loss,  the non-optimal hypotheses are not penalized. Thus, they can support any artificial data regions without being informed via the learning signal; for a more formal discussion see the Appendix. We refer to this problem as the inconsistency of the model regarding the real underlying data distribution.

As a new alternative, we propose adding a discriminator D as a critic when training the multiple-hypotheses autoencoder G; see Fig. \ref{fig:discriminator training}. 
D and G are optimized together on the minimax loss 
    \begin{dmath}
    \min_D \max_G  L_{D}(x,z) = \min_D \max_G  \underbrace{- \log(p_D(x_{real}))}_{L_{real}} + L_{fake}(x,z)
    \label{eq:LD}
    \end{dmath}
        \begin{dmath}
\text{with } {L_{fake}(x,z) =   \log(p_D(\hat{x}_{z\sim \mathcal{N}(0,1)}))}  +  {\log(p_D(\hat{x}_{z\sim \mathcal{N}(\mu_{z|x},\Sigma_{z|x})}))  + \log(p_D(\hat{x}_{best-guess}))}
    \label{eq:LFake}
    \end{dmath}
Figure \ref{fig:discriminator training} illustrates how samples are fed into the discriminator.  In contrast to a standard GAN, samples labeled as fake come from three different sources: randomly-sampled images $\hat{x}_{z\sim \mathcal{N}(0,1)}$, data reconstruction defined by individual hypotheses $\hat{x}_{z\sim \mathcal{N}(\mu_{z|x},\Sigma_{z|x})}$, the best combination of hypotheses according to the winner-takes-all loss $\hat{x}_{\text{best\_guess}}$.

Accordingly, the learning objective for the VAE generator becomes:
    \begin{equation}
    \min_G L_G = \min_G L_{WTA} +  KL(q_{\phi}(z|x)||\mathcal{N}(0,1)) - L_{D},
    \label{Eq:generator loss}
    \end{equation}
where KL denotes the symmetrized Kullback-Leibler divergence (Jensen-Shannon divergence). Intuitively, the discriminator enforces the generated hypotheses to remain in realistic data regions. The model is trained until the WTA-loss is minimized on the validation set.

\subsection{Avoiding mode collapse} 

To avoid mode collapse of the discriminator training and hypotheses, we propose to employ hypotheses discrimination. This is inspired by minibatch discrimination \citep{salimans2016improved}. Concretely, in each batch, the discriminator receives the pair-wise features-distance of generated hypotheses. Since batches of real images have large pair-wise distances, the generator has to generate diverse outputs to avoid being detected too easily. Training with hypotheses discrimination naturally leads to more diversity among hypotheses.

Fig. \ref{fig:diversity-is-better} shows a simple example of why more diversity among hypotheses is beneficial. The hypotheses correspond to cluster centers in the image-conditional space. Maximizing diversity among hypotheses is, hence, similar to the maximization of inter-class-variance in typical clustering algorithm such as Linear Discriminant Analysis \citep{mika1999fisher}.

\begin{figure}[t]
    \centering
     \begin{subfigure}{.15\textwidth}
     \centering
      \includegraphics[width=1.0\linewidth]{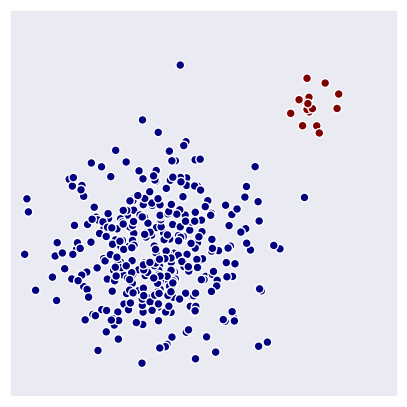}
      \caption{}
      \label{fig:sfig1Diversity}
    \end{subfigure}%
    \begin{subfigure}{.15\textwidth}
    \centering
      \includegraphics[width=1.0\linewidth]{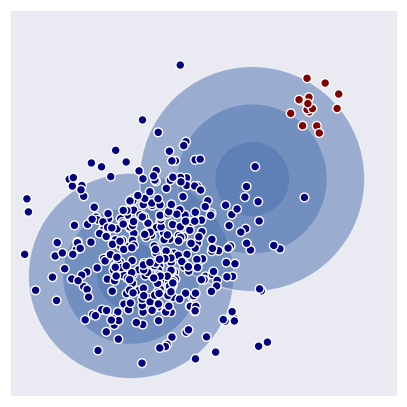}
      \caption{}
      \label{fig:sfig2Diversity}
    \end{subfigure}
    \begin{subfigure}{.15\textwidth}
        \centering
      \includegraphics[width=1.0\linewidth]{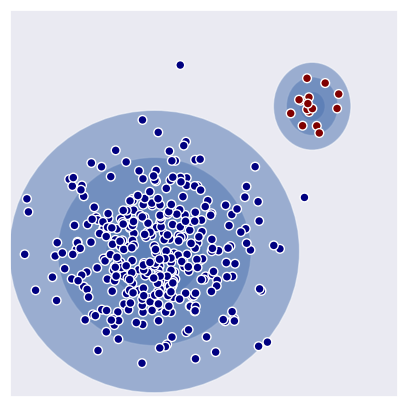}
      \caption{}
      \label{fig:sfig3Diversity}
    \end{subfigure}
     \caption{(a) Modeling task with one extremely dominant data mode (dense region) and one under-represented mode. (b) shows how multiple-hypotheses predictions are used to cover data modes. Hypotheses tend to concentrate on the dominant mode, which leads to over-fitting in this region. (c) Increasing diversity across hypotheses (similar to  maximizing inter-class variance) leads to better coverage of the underlying data.}
    \label{fig:diversity-is-better}
\end{figure}

\subsection{Anomaly score estimation based on local neighborhood}
\label{ana: local neighborhood}

Hypotheses are spread out to cover the data modes seen during training. Due to the loose coupling between hypotheses, the probability mass of each hypothesis is only distributed \emph{within the respective cluster}. Compared to traditional likelihood learning, the conditional probability mass only sums up to $1$ within each hypothesis branch, i.e., the combination of all hypotheses does not yield a proper density function as in mixture density networks. 
However, we can use the winner-takes-all loss as the pixel-wise sample anomaly score. Hence, each pixel likelihood is only evaluated based on the best-matching conditional hypothesis. We refer to this as anomaly detection based on local likelihood estimation. 

\paragraph{Local likelihood is more effective for anomaly score estimation}

Fig. \ref{fig:MDN-LOF-MHP} provides an intuition, why the local neighborhood is more effective in anomaly detection. The red point represents a new normal point  which is very close to one less dominant data mode. By using the global likelihood function (Fig. \ref{fig:sfig3LOF}), the anomaly score depends on all other points.

However, samples further away intuitively do not affect the anomaly score estimation. In Local-outlier-factor \citep{breunig2000lof}, outlier score estimation only depends on samples close to the new point (fig. \ref{fig:sfig4LOF}).  Similarly, our multi-hypotheses model considers only the next cluster (fig. \ref{fig:sfig5LOF}) and provides a more accurate anomaly score.

Further, learning local likelihood estimations is easier and more sample-efficient than learning from a global likelihood function, since the local model need not learn the global dependencies. During training, it is sufficient if samples are covered by at least one hypothesis. 

In summary, we estimate the anomaly scores based on the consistency of new samples regarding the closest hypotheses. Accordingly, we refer to our framework as \emph{consistency-based anomaly detection (ConAD)}.

\section{Related works}%
In high-dimensional input domains such as images, modern generative models \citep{kingma2013auto, goodfellow2014generative} are typically used to learn the data distribution for the normal data \citep{cong2011sparse,li2014anomaly,ravanbakhsh2017abnormal}. In many cases, anomaly detection might improve the models behavior in out-of-distribution cases \citep{nguyen2018multisource}.

For learning in uncertain tasks,  \citet{koltun_multi_choise,bhattacharyya2018accurate,rupprecht_learning_2016,ilg_uncertainty_2018} independently proposed multiple-hypotheses-predictions (MHP) networks. More details about theses works can be found in the Appendix.

In contrast to previous MHP-networks, we propose to utilize these networks for anomaly detection for the first time. To this end, we introduce a strategy to avoid the support of artificial data modes, namely via a discriminator as a critic. 
\cite{rupprecht_learning_2016} suggested a soft WTA-loss, where the non-optimal hypotheses receive a small fraction of the learning signal. Depending on the softening parameter $\epsilon$, the model training results in a state between mean-regression (i.e., uni-modal learning) and large support of non-existing data modes (more details in the Appendix). Therefore, the soft-WTA-loss is a compromise of contradicting concepts and, thus, requires a good choice of the corresponding hyperparameter. 
In the case of anomaly detection, the hyperparameter search cannot be formalized, since there are not enough anomalous data points available.

Compared to previous reconstruction-based anomaly detection methods  (using, e.g., \citet{kingma2013auto, bishop1994mixture}), our framework evaluates anomaly score only based on the local instead of the global neighborhood. Further, the model learns from a relaxed version of likelihood maximizing, which results in better sample efficiency.

\section{Experiments}

    In this section, we compare the proposed approach to previous deep learning and non-deep learning techniques for one-class learning tasks. 
    Since true anomaly detection benchmarks are rare, we first tested on CIFAR-10, where one class is used as the normal case to be modeled, and the other 9 classes are considered as anomalies and are only available at test time. Besides, we tested on a true anomaly detection task on a metal anomaly dataset, where arbitrary deviations from the normal case can appear in the data. 
    \subsection{Network architecture} 
        The networks are following DCGAN \citep{radford_unsupervised_2015} but were scaled down to support the low-resolution of CIFAR-10. Concretely, the decoder only uses a sequence of Dense-Deconv.-Conv.-Deconv. layers and on top, $2*n$ Deconv. layer for $n$ hypotheses branches. Each branch requires two layers since for each pixel position, the network predicts a $\mu$ and$\sigma$ for the conditional distribution.  Further, throughout the network, leaky-relu units are employed. 
        
        Hypotheses branches are represented as decoder networks heads. Each hypothesis predicts one Gaussian distribution with diagonal co-variance $\Sigma$ and means $\mu$. The winner-takes-all loss operates on the pixel-level, i.e., for each predicted pixel, there is a single winner across hypotheses. The best-combined-reconstructions is the combination of the winning hypotheses on pixel-level.
        
        \subsection{Training }
        For training with the discriminator in Fig. \ref{fig:discriminator training}, samples are forwarded separately through the network. The batch-size $n$ was set to 64 each on CIFAR-10, 32 on the Metal Anomaly dataset. Adam \citep{kingma2014adam} was used for training with a learning rate of 0.001. Per discriminator training, the generator is trained at most five epochs to balance both players.  We use the validation set of samples from the normal class to early stop the training if no better model regarding the corresponding loss could be found.

    \begin{table}[t]
    \caption{Dataset description. CIFAR-10 is transformed into 10 anomaly detection tasks, where one class is used as the normal class, and the remaining classes are treated as anomalies. The train \& validation datasets contain only samples from the normal class. This scenario resembles the typical situation where anomalies are extremely rare and not available at training time, as in the Metal Anomaly dataset.}
   \vskip 0.15in
\begin{center}
\begin{small}
\begin{sc} 
    \centering
        \begin{tabular}[t]{lp{.6cm}cc}%
        \toprule
        \midrule
&Type &CIFAR-10  &Metal anomaly \\
\cmidrule{2-4}
Problem &\hspace{.35cm}- &   1 vs. 9  &1 vs. 1 \\
Tasks & \hspace{.35cm}- &  10 &1 \\
Resolution & \hspace{.35cm}- &32x32 &224x224 \\
\cmidrule{2-4}
\multirow{3}{*}{Normal data} &  Train &  4500 &5408 \\
                             &  Valid &   500  &1352 \\
                             &  Test &    1000 &1324 \\
\cmidrule{2-4}
Anomaly &  Test & 9000 & 346 \\
        \midrule
        \bottomrule
        \end{tabular} 
        \label{Tab: Dataset}
        \end{sc}
        \end{small}
        \end{center}
    \end{table}

    \subsection{Evaluation }
    \paragraph{Experiments details} Quantitative evaluation is done on CIFAR-10 and the Metal Anomaly dataset  (Tab.\ref{Tab: Dataset}). The typical 10-way classification task in CIFAR-10 is transformed into 10 one vs. nine anomaly detection tasks. Each class is used as the normal class once; all remaining classes are treated as anomalies.  During model training, only data from the normal data class is used, data from anomalous classes are abandoned. At test time, anomaly detection performance is measured in Area-Under-Curve of Receiver Operating Curve (AUROC) based on normalized negative log-likelihood scores given by the training objective. 

    In Tab. \ref{Tab:CIFAR-10-summary},  we evaluated on CIFAR-10 variants of our multiple-hypotheses approaches including the following energy formulations: MDN \citep{bishop1994mixture}, MHP-WTA \citep{ilg_uncertainty_2018}, MHP \citep{rupprecht_learning_2016}, ConAD, and MDN+GAN. We compare our methods against vanilla VAE \citep{kingma2013auto, rezende2014stochastic} , VAEGAN \citep{larsen2015autoencoding,dosovitskiy2016generating}, AnoGAN \citep{schlegl2017unsupervised}, AdGAN \citealp{deecke2018anomaly}, OC-Deep-SVDD \citep{pmlr-v80-ruff18a}. Traditional approaches considered are: Isolation Forest \citep{liu2008isolation,liu2012isolation}, OCSVM \citep{scholkopf2001estimating}. The performance of traditional methods suffers due to the curse of dimensionality \citep{zong2018deep}. 
    
    Furthermore, on the high-dimensional Metal Anomaly dataset, we focus only on the evaluation of deep learning techniques. The GAN-techniques proposed by previous work AdGAN \& AnoGAN heavily suffer from instability due to pure GAN-training on a small dataset. Hence, their training leads to random anomaly detection performance. Therefore, we only evaluate MHP-based approaches against their uni-modal counterparts (VAE, VAEGAN).

\paragraph{Anomaly detection on CIFAR-10}

    Tab. \ref{Tab:CIFAR-10-all} and Tab. \ref{tab:Effect of hypotheses numbers on different models} show an extensive evaluation of different traditional and deep learning techniques. Results are adopted from \cite{deecke2018anomaly} in which the training and testing scenarios were similar. The average performance overall 10 anomaly detection tasks are summarized in Tab. \ref{Tab:CIFAR-10-summary}.
\begin{table}[h]
\caption{Anomaly detection on CIFAR-10, performance measured in AUROC. Each class is considered as the normal class once with all other classes being considered as anomalies, resulting in 10 one-vs-nine classification tasks. Performance is averaged for all ten tasks and over three runs each (see Appendix for detailed performance).  Our approach significantly outperforms previous non-Deep Learning and Deep Learning methods.
        }
   \vskip 0.15in
\centering
\begin{small}
\begin{sc} 
\begin{tabular}{p{1.1cm}C{1.5cm}C{1.5cm}C{.90cm}c  }
    \toprule
    \midrule
 Type    & \multicolumn{4}{c} {Models}\\
 \midrule
 \multirow{2}{*}{Non-DL.}  & KDE-PCA  & OC-SVM-PCA & IF & GMM \\
 & 59.0 &   61.0 & 55.8   &   58.5 \\
 \cmidrule{2-5}
 \multirow{2}{*}{DL}&. AnoGAN & OC-D-SVDD  & ADGAN & ConAD \\
& 61.2 & 63.2 & 62.0  & \textbf{67.1}\\
\midrule
    \bottomrule
\end{tabular} 
\label{Tab:CIFAR-10-summary}
    \end{sc}
    \end{small}
    \end{table}    
    Traditional, non-deep-learning methods only succeed to capture classes with a dominant homogeneous background such as ships, planes, frogs (backgrounds are water, sky, green nature respectively). This issue occurs due to preceding feature projection with PCA, which focuses on dominant axes with large variance. \cite{deecke2018anomaly} reported that even  features from a pretrained AlexNet have no positive effect on anomaly detection performance.

    Our approach ConAD outperforms previously reported results by 3.9\% absolute improvement. Furthermore, compared to other multiple-hypotheses-approaches (MHP, MDN, MHP+WTA), our model could benefit from the increased capacity given by the additional hypotheses. The combination of discriminator training and a high number of hypotheses is crucial for high detection performance as indicated in our ablation study (Tab. \ref{tab:Ablation study}).

    \begin{table*}[h]
    \centering
    \caption{CIFAR-10 anomaly detection: AUROC-performance of different approaches. The column indicates which class was used as in-class data for distribution learning. Note that random performance is at 50\% and higher scores are better. Top-2-methods are marked. Our ConAD approach outperforms traditional methods and vanilla MHP-approaches significantly and can benefit from an increasing number of hypotheses.
        }
        \vskip .15in
        \begin{small}
        \begin{sc}
            \begin{tabular}{l||ccccccccccc||c}%
            CIFAR-10 & 0 & 1 & 2 & 3 & 4 & 5 & 6 & 7 & 8 & 9 & Mean\\
            \hline

            VAE   & 77.1  & 46.7   & 68.4   & 53.8   & 71.    & 54.2   & 64.2   & 51.2   & \textbf{76.5}   & 46.7   & 61.0\\ 
            OC-D-SVDD & 61.7& 65.9& 50.8& 59.1& 60.9& 65.7& 67.7& 67.3& 75.9& 73.1 & 63.2\\

            \hline
            
            MDN-2 & 76.1 &  46.9 &  68.7 &  53.8 &  70.4 &  53.8 &  63.2 &  52.3 &  \textbf{76.8} & 46.7  & 60.9 \\
            MDN-4 & 76.9 &  46.8 &  68.6 &  53.5 &  69.3 &  54.4 &  63.5 &  54.1 &  76.  & 46.9  & 61.0 \\
            MDN-8 & 76.2 &  46.9 &  68.6 &  53.3 &  70.4 &  54.7 &  63.3 &  53.  &  76.3 & 47.3  & 61. \\
            MDN-16 & 76.2 &  47.9 &  68.2 &  52.8 &  70.1 &  54.  &  63.5 &  52.9 &  76.4 & 46.9  & 60.9 \\
            \hline
            MHP-WTA-2 & 77.3 &  51.6 &  68.  &  55.2 &  69.5 &  54.3 &  64.3 &  55.5 &  76.  & 51.2  & 62.2 \\
            MHP-WTA-4 & \textbf{77.8} &  53.9 &  65.1 &  56.7 &  66.  &  54.2 &  63.5 &  56.3 &  75.2 & 54.1  & 62.2 \\
            MHP-WTA-8 & 76.1 &  56.  &  62.7 &  58.8 &  62.6 &  55.3 &  61.4 &  57.8 &  74.3 &  54.8 & 61.9 \\
            MHP-WTA-16 & 75.7 &  56.7 &  60.9 &  59.8 &  62.7 &  56.  &  61.  &  56.8 &  73.8 & 57.3  & 62. \\
            \hline
            MHP-2 & 75.5 &  49.9 &  67.6 &  54.6 &  69.3 &  54.3 &  63.6 &  57.7 &  76.4 & 50.8  & 61.9 \\
            MHP-4 & 75.2 &  51.  &  66.  &  56.8 &  67.7 &  55.1 &  64.4 &  56.  &  76.4 & 51.   & 61.9 \\
            MHP-8 & 75.7 &  54.  &  65.2 &  57.6 &  64.8 &  55.4 &  62.5 &  54.7 &  75.9 & 53.   & 61.8 \\
            MHP-16 & 75.8 &  53.9 &  64.1 &  58.5 &  64.6 &  55.2 &  62.3 &  54.5 &  75.9 & 53.2  & 61.7 \\
            \hline
            MDN+GAN-2 & 74.6 &  48.9 &  68.6 &  52.1 &  71.1 &  52.5 &  66.8 &  57.7 &  76.5 & 48.1  & 61.6 \\
            MDN+GAN-4 & 76.2 &  50.4 &  69.  &  52.4 &  71.6 &  53.2 &  65.9 &  58.3 &  75.3 & 48.9  & 62.1 \\
            MDN+GAN-8 & 77.4 &  48.3 &  69.3 &  53.1 &  72.2 &  53.7 &  67.9 &  54.  &  76.  & 51.9  & 62.3 \\
            MDN+GAN-16 & 73.6 &  46.9 &  69.4 &  52.2 &  75.3 &  54.1 &  65.7 &  56.8 &  75.3 & 45.4  & 61.4 \\
            \hline

            \hline
            
            ConAD - 2 (ours) & 77.3 &    60.0 &  66.6 &    56.2 &    69.4 &    56.1 &    70.6 &    63.0 &     74.8 &    49.9 &    64.3 \\
            ConAD - 4 (ours) & \textbf{77.6} &    52.5 &    66.3 &    57.0 &     68.7 &    54.1 &    \textbf{80.1} &    54.8 &    74.1 &    53.9 &    {63.9} \\
            ConAD - 8 (ours) & 77.4 &    \textbf{65.2} &    64.8 &    60.1 &    67.0 &     {57.9} &     72.5 &    \textbf{66.2} &    74.8 &    \textbf{66.0} &     \textbf{67.1} \\
            ConAD - 16 (ours)& 77.2& \textbf{63.1}& 63.1& \textbf{61.5}& 63.3& \textbf{58.8}& 69.1& \textbf{64.0}& 75.5& \textbf{63.7}& \textbf{65.9}\\
            
        \end{tabular} 
        \end{sc}
        \end{small}

        \label{Tab:CIFAR-10-all}
    \end{table*}

\begin{figure*}[h]
\centering
        \begin{subfigure}{.12\linewidth}
            \centering
      \includegraphics[width=.97\linewidth]{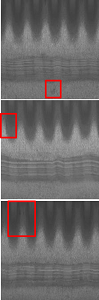}
      \caption{}
      \label{fig:teaser-sfig11}
    \end{subfigure}%
    \begin{subfigure}{.27\textwidth}
      \centering
      \includegraphics[width=.95\linewidth]{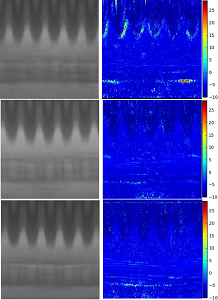}
      \caption{}
      \label{fig:teaser-sfig21}
    \end{subfigure}
    \begin{subfigure}{.265\textwidth}
      \includegraphics[width=.97\linewidth]{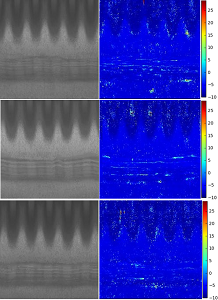}
      \caption{}
      \label{fig:sfig31}
    \end{subfigure}
    \caption{(a) anomalous samples on Metal Anomaly data-set. Anomalies are highlighted. (b) shows maximum-likelihood reconstructions under a Variational Autoencoder and the corresponding anomaly heatmaps based on negative-log-likelihood. (c) shows the reconstructions and anomaly maps for ConAD. In all cases, the maximum-likelihood expectation under the unimodal model is blurry and should itself be seen as an anomaly. Contrary, under our model, the maximum-likelihood expectation of the input is much closer to the input and more realistic. Due to the fine-grained learning, the anomaly heatmaps could reliably identify the location and strength of possible anomalies.}
    \label{fig:more Metal_anomaly}
\end{figure*}

 \begin{table}[h]
      \caption{Anomaly detection performance on CIFAR-10 dependent on multiple-hypotheses-predictions models and hypotheses number. Performance averaged over tasks and in multiple runs each.}
\vskip .15in
\centering
 \begin{small}
 \begin{sc}
 \begin{tabular}{l C{2cm} *{4}{p{.48cm}}}
    \toprule
    \midrule
    & \multicolumn{5}{c}{Hypotheses}\\
    \cmidrule{2-6}
    Models & 1 & 2 & 4 & 8 & 16 \\ 
    \midrule
    MHP & \multirow{3}{*}{61.0 =VAE} & 61.9 & 61.9  & 61.8&  61.7\\ 
    MHP+WTA &  & 62.2 & 62.2  & 61.9 & 62.0\\ 
    MDN &  &  60.9& 61.0 &   61.0&  60.9\\ 
    \cmidrule{2-6}
    MDN+GAN & \multirow{2}{*}{61.7 =VAEGAN} &  61.6& 62.1 &  62.0 & 61.4 \\ 
    ConAD &  & \textbf{64.3} & \textbf{63.9} & \textbf{67.1}& \textbf{65.9}\\ 
    \midrule
    \bottomrule
    \label{tab:Effect of hypotheses numbers on different models}
    \end{tabular}
          
 \end{sc}
 \end{small}
 \end{table}  
 
    \begin{table}[h]
    \centering
    \caption{Ablation study of our approach ConAD on CIFAR-10, meausured in anomaly detection performance (AUROC-scores on unseen   contaminated dataset). }     
    \vskip 0.15in 
    \begin{small}
    \begin{sc}
        \begin{tabular}{l c}
        \toprule
        \midrule
        Configuration & AUROC\\
        \midrule
         ConAD (8-hypotheses)&  \textbf{67.1}\\
         - fewer hypotheses (2) & 64.3\\
         - Discriminator&  61.9\\
         - Winner-takes-all-loss (WTA)& 61.8\\
         - WTA \& loose hyp. coupling & 61.0\\
         - Multiple-hypotheses & 61.7\\
         - Multiple-hypotheses \& discriminator  &  61.0 \\
         \midrule
         \bottomrule
    \end{tabular}
    \label{tab:Ablation study}
    \end{sc}
    \end{small}
\end{table}

    \paragraph{Anomaly detection on Metal Anomaly dataset}
    Fig. \ref{fig:more Metal_anomaly} shows a qualitative analysis of uni-modal learning with  VAE \citep{kingma2013auto} compared to  our framework ConAD. Due to the fine-grained learning with multiple-hypotheses, our maximum-likelihood reconstructions of samples are  significantly closer to the input. Contrary, VAE training results in blurry reconstructions and hence falsified anomaly heatmaps, hence  cannot separate possible anomaly from dataset details.

    Tab. \ref{tab:a} shows an evaluation of MHP-methods against multi-modal density-learning methods such as MDN \citep{bishop1994mixture}, VAEGAN \citep{dosovitskiy2016generating,larsen2015autoencoding}. Note that the VAE-GAN model corresponds to our ConAD with a single hypothesis. The VAE corresponds to a single hypothesis variant of MHP, MHP-WTA, and MDN.

    \begin{table}[h]
       \vskip 0.15in
        \centering
        \caption{Anomaly detection performance and their standard variance on the Metal Anomaly dataset. To reduce noisy residuals due to the high-dimensional input domain, only 10\% of maximally abnormal pixels with the highest residuals are summed to form the total anomaly score. AUROC is computed on an unseen test set, a combination of normal and anomaly data. For more detailed results see Appendix. The anomaly detection performance of plain MHP rapidly breaks down with an increasing number of hypotheses. }
        \vskip 0.15in

        \begin{small}
        \begin{sc} 
        \begin{tabular}{p{1.5cm}p{.50cm}ccc }
        \toprule
        \midrule
         & \multicolumn{4}{c}{Hypotheses}\\
         \cmidrule{2-5}
     Model & 1 & 2 & 4 & 8  \\ 
    \midrule

MHP & \multirow{3}{*}{\shortstack[l]{94.2 \\(1.4)}}  & 98.0  (0.5)& 97.0  (1.0)& 95.0  (0.2)\\
MHP+WTA &  & 98.0  (0.9)& \textbf{98.0} (0.1)& 94.6  (3.3)\\
MDN & & 90.0  (1.1)& 91.0  (1.9)& 91.6  (3.5)\\
\cmidrule{2-5}
MDN+GAN & \multirow{2}{*}{\shortstack[l]{93.6 \\(0.7)}} & 94.2  (1.6)& 91.3  (1.9) &  94.3  (1.1)\\
ConAD & & \textbf{98.5} (0.1) &  97.7 (0.5)&  \textbf{96.5} (0.2)\\
    \midrule
    \bottomrule
        \end{tabular}
        \end{sc}
        \end{small}
        \label{tab:a}
\end{table}

    The significant improvement of up to 4.2\% AUROC-score comes from the loose coupling of hypotheses in combination with a discriminator D as quality assurance. In a high-dimensional domain such as images, anomaly detection with MDN is worse than MHP approaches. This result from  (1) typical mode collapse in MDN and (2) global neighborhood consideration for anomaly score estimation. 
    Using the MHP-technique, better performance is already achieved with two hypotheses. However, without the discriminator D,  an increasing number of hypotheses rapidly leads to performance breakdown, due to the inconsistency property of generated hypotheses. Intuitively, additional non-optimal hypotheses are not strongly penalized during training, if they support artificial data regions. 
    With our framework ConAD, anomaly detection performance remains competitive or better even with an increasing number of hypotheses available. The discriminator D makes the framework adaptable to the new dataset and less sensitive to the number of hypotheses to be used.

    When more hypotheses are used (8), the anomaly detection performance in all multiple-hypotheses models rapidly breaks down. The standard variance of performance of standard approaches remains high (up to $\pm$ 3.5). The reason might be the beneficial start for some hypotheses branches, which adversely affect non-optimal branches.

    This effect is less severe in our framework ConAD. The standard variance of our approaches is also significantly lower. We suggest that the noise is then learned too easily. Consider the extreme case when there are 255 hypotheses available. The winner-takes-all loss will encourage each hypothesis branch to predict a constant image with one value from [0,255]. In our framework, the discriminator as a critic attempts to alleviate this effect.   That might be a reason why our ConAD has less severe performance breakdown. Our model ConAD is less sensitive to the choice of the hyper-parameter for the number of hypotheses. It  enables better exploitation of the additional expressive power provided by the MHP-technique for new anomaly detection tasks.
    
    Our method can detect more subtle anomalies due to the focus on extremely similar samples in the local neighborhood. However, the added capacity by the hypotheses branches makes the network more sensitive to  large  label noise in the datasets. Hence, robust anomaly detection under label noise is a possible future research direction.
    \section{Conclusion}
    
    In this work, we propose to employ multiple-hypotheses networks for learning data distributions for anomaly detection tasks. 
    Hypotheses are meant to form clusters in the data space and can easily capture model uncertainty not encoded by the latent code. Multiple-hypotheses networks can provide a more fine-grained description of the data distribution and therefore enable also a more fine-grained anomaly detection.
    Furthermore, to reduce support of artificial data modes by hypotheses learning, we propose using a discriminator D as a critic.  
    The combination of multiple-hypotheses learning with D aims to retain the consistency of estimated data modes w.r.t. the real data distribution. Further, D encourages diversity across hypotheses with hypotheses discrimination. Our framework allows the model to identify out-of-distribution samples reliably. 
    
    For the anomaly detection task on CIFAR-10, our proposed model results in up to 3.9\% points improvement over previously reported results. On a real anomaly detection task, the approach reduces the error of the baseline models from 6.8\% to 1.5\%.

\newpage

\section*{Acknowledgements}
  This research was supported by Robert Bosch GmbH. We thank our colleagues Oezguen Cicek, Thi-Hoai-Phuong Nguyen and the four anonymous reviewers who provided great feedback and their expertise to improve our work.

\bibliography{iclr2018_conference}

\begin{thebibliography}{31}
\providecommand{\natexlab}[1]{#1}
\providecommand{\url}[1]{\texttt{#1}}
\expandafter\ifx\csname urlstyle\endcsname\relax
  \providecommand{\doi}[1]{doi: #1}\else
  \providecommand{\doi}{doi: \begingroup \urlstyle{rm}\Url}\fi

\bibitem[Bhattacharyya et~al.(2018)Bhattacharyya, Schiele, and
  Fritz]{bhattacharyya2018accurate}
Bhattacharyya, A., Schiele, B., and Fritz, M.
\newblock Accurate and diverse sampling of sequences based on a “best of
  many” sample objective.
\newblock In \emph{Proceedings of the IEEE Conference on Computer Vision and
  Pattern Recognition}, pp.\  8485--8493, 2018.

\bibitem[Bishop(1994)]{bishop1994mixture}
Bishop, C.~M.
\newblock Mixture density networks.
\newblock Technical report, Citeseer, 1994.

\bibitem[Breunig et~al.(2000)Breunig, Kriegel, Ng, and Sander]{breunig2000lof}
Breunig, M.~M., Kriegel, H.-P., Ng, R.~T., and Sander, J.
\newblock Lof: identifying density-based local outliers.
\newblock In \emph{ACM sigmod record}, volume~29, pp.\  93--104. ACM, 2000.

\bibitem[Chen \& Koltun(2017)Chen and Koltun]{koltun_multi_choise}
Chen, Q. and Koltun, V.
\newblock Photographic image synthesis with cascaded refinement networks.
\newblock In \emph{{IEEE} International Conference on Computer Vision, {ICCV}
  2017, Venice, Italy, October 22-29, 2017}, pp.\  1520--1529, 2017.

\bibitem[Cong et~al.(2011)Cong, Yuan, and Liu]{cong2011sparse}
Cong, Y., Yuan, J., and Liu, J.
\newblock Sparse reconstruction cost for abnormal event detection.
\newblock In \emph{CVPR 2011}, pp.\  3449--3456. IEEE, 2011.

\bibitem[Deecke et~al.(2018)Deecke, Vandermeulen, Ruff, Mandt, and
  Kloft]{deecke2018anomaly}
Deecke, L., Vandermeulen, R., Ruff, L., Mandt, S., and Kloft, M.
\newblock Anomaly detection with generative adversarial networks.
\newblock 2018.

\bibitem[Dey et~al.(2015)Dey, Ramakrishna, Hebert, and
  Andrew~Bagnell]{dey2015predicting}
Dey, D., Ramakrishna, V., Hebert, M., and Andrew~Bagnell, J.
\newblock Predicting multiple structured visual interpretations.
\newblock In \emph{Proceedings of the IEEE International Conference on Computer
  Vision}, pp.\  2947--2955, 2015.

\bibitem[Dosovitskiy \& Brox(2016)Dosovitskiy and
  Brox]{dosovitskiy2016generating}
Dosovitskiy, A. and Brox, T.
\newblock Generating images with perceptual similarity metrics based on deep
  networks.
\newblock In \emph{Advances in Neural Information Processing Systems}, pp.\
  658--666, 2016.

\bibitem[Goodfellow et~al.(2014)Goodfellow, Pouget-Abadie, Mirza, Xu,
  Warde-Farley, Ozair, Courville, and Bengio]{goodfellow2014generative}
Goodfellow, I., Pouget-Abadie, J., Mirza, M., Xu, B., Warde-Farley, D., Ozair,
  S., Courville, A., and Bengio, Y.
\newblock Generative adversarial nets.
\newblock In \emph{Advances in neural information processing systems}, pp.\
  2672--2680, 2014.

\bibitem[Ilg et~al.(2018)Ilg, {\c{C}}i{\c{c}}ek, Galesso, Klein, Makansi,
  Hutter, and Brox]{ilg_uncertainty_2018}
Ilg, E., {\c{C}}i{\c{c}}ek, {\"O}., Galesso, S., Klein, A., Makansi, O.,
  Hutter, F., and Brox, T.
\newblock Uncertainty {Estimates} with {Multi}-{Hypotheses} {Networks} for
  {Optical} {Flow}.
\newblock In \emph{European Conference on Computer Vision (ECCV)}, 2018.
\newblock URL
  \url{http://lmb.informatik.uni-freiburg.de/Publications/2018/ICKMB18}.
\newblock https://arxiv.org/abs/1802.07095.

\bibitem[Kingma \& Ba(2014)Kingma and Ba]{kingma2014adam}
Kingma, D.~P. and Ba, J.
\newblock Adam: A method for stochastic optimization.
\newblock \emph{arXiv preprint arXiv:1412.6980}, 2014.

\bibitem[Kingma \& Welling(2013)Kingma and Welling]{kingma2013auto}
Kingma, D.~P. and Welling, M.
\newblock Auto-encoding variational bayes.
\newblock \emph{arXiv preprint arXiv:1312.6114}, 2013.

\bibitem[Larsen et~al.(2015)Larsen, S{\o}nderby, Larochelle, and
  Winther]{larsen2015autoencoding}
Larsen, A. B.~L., S{\o}nderby, S.~K., Larochelle, H., and Winther, O.
\newblock Autoencoding beyond pixels using a learned similarity metric.
\newblock \emph{arXiv preprint arXiv:1512.09300}, 2015.

\bibitem[Lee et~al.(2017)Lee, Hwang, Park, and Shin]{lee2017confident}
Lee, K., Hwang, C., Park, K., and Shin, J.
\newblock Confident multiple choice learning.
\newblock \emph{arXiv preprint arXiv:1706.03475}, 2017.

\bibitem[Lee et~al.(2016)Lee, Prakash, Cogswell, Ranjan, Crandall, and
  Batra]{lee2016stochastic}
Lee, S., Prakash, S. P.~S., Cogswell, M., Ranjan, V., Crandall, D., and Batra,
  D.
\newblock Stochastic multiple choice learning for training diverse deep
  ensembles.
\newblock In \emph{Advances in Neural Information Processing Systems}, pp.\
  2119--2127, 2016.

\bibitem[Li et~al.(2014)Li, Mahadevan, and Vasconcelos]{li2014anomaly}
Li, W., Mahadevan, V., and Vasconcelos, N.
\newblock Anomaly detection and localization in crowded scenes.
\newblock \emph{IEEE transactions on pattern analysis and machine
  intelligence}, 36\penalty0 (1):\penalty0 18--32, 2014.

\bibitem[Liu et~al.(2008)Liu, Ting, and Zhou]{liu2008isolation}
Liu, F.~T., Ting, K.~M., and Zhou, Z.-H.
\newblock Isolation forest.
\newblock In \emph{2008 Eighth IEEE International Conference on Data Mining},
  pp.\  413--422. IEEE, 2008.

\bibitem[Liu et~al.(2012)Liu, Ting, and Zhou]{liu2012isolation}
Liu, F.~T., Ting, K.~M., and Zhou, Z.-H.
\newblock Isolation-based anomaly detection.
\newblock \emph{ACM Transactions on Knowledge Discovery from Data (TKDD)},
  6\penalty0 (1):\penalty0 3, 2012.

\bibitem[Mika et~al.(1999)Mika, Ratsch, Weston, Scholkopf, and
  Mullers]{mika1999fisher}
Mika, S., Ratsch, G., Weston, J., Scholkopf, B., and Mullers, K.-R.
\newblock Fisher discriminant analysis with kernels.
\newblock In \emph{Neural networks for signal processing IX, 1999. Proceedings
  of the 1999 IEEE signal processing society workshop.}, pp.\  41--48. Ieee,
  1999.

\bibitem[Nguyen et~al.(2018)Nguyen, Spehr, Zug, and
  Kruse]{nguyen2018multisource}
Nguyen, T.~T., Spehr, J., Zug, S., and Kruse, R.
\newblock Multisource fusion for robust road detection using online estimated
  reliabilities.
\newblock \emph{IEEE Transactions on Industrial Informatics}, 14\penalty0
  (11):\penalty0 4927--4939, 2018.

\bibitem[Radford et~al.(2015)Radford, Metz, and
  Chintala]{radford_unsupervised_2015}
Radford, A., Metz, L., and Chintala, S.
\newblock Unsupervised {Representation} {Learning} with {Deep} {Convolutional}
  {Generative} {Adversarial} {Networks}.
\newblock \emph{arXiv:1511.06434 [cs]}, November 2015.
\newblock URL \url{http://arxiv.org/abs/1511.06434}.
\newblock arXiv: 1511.06434.

\bibitem[Ravanbakhsh et~al.(2017)Ravanbakhsh, Nabi, Sangineto, Marcenaro,
  Regazzoni, and Sebe]{ravanbakhsh2017abnormal}
Ravanbakhsh, M., Nabi, M., Sangineto, E., Marcenaro, L., Regazzoni, C., and
  Sebe, N.
\newblock Abnormal event detection in videos using generative adversarial nets.
\newblock In \emph{2017 IEEE International Conference on Image Processing
  (ICIP)}, pp.\  1577--1581. IEEE, 2017.

\bibitem[Rezende et~al.(2014)Rezende, Mohamed, and
  Wierstra]{rezende2014stochastic}
Rezende, D.~J., Mohamed, S., and Wierstra, D.
\newblock Stochastic backpropagation and approximate inference in deep
  generative models.
\newblock \emph{arXiv preprint arXiv:1401.4082}, 2014.

\bibitem[Ruff et~al.(2018)Ruff, Vandermeulen, Goernitz, Deecke, Siddiqui,
  Binder, M{\"u}ller, and Kloft]{pmlr-v80-ruff18a}
Ruff, L., Vandermeulen, R., Goernitz, N., Deecke, L., Siddiqui, S.~A., Binder,
  A., M{\"u}ller, E., and Kloft, M.
\newblock Deep one-class classification.
\newblock In Dy, J. and Krause, A. (eds.), \emph{Proceedings of the 35th
  International Conference on Machine Learning}, volume~80 of \emph{Proceedings
  of Machine Learning Research}, pp.\  4393--4402, Stockholmsmässan, Stockholm
  Sweden, 10--15 Jul 2018. PMLR.
\newblock URL \url{http://proceedings.mlr.press/v80/ruff18a.html}.

\bibitem[Rupprecht et~al.(2016{\natexlab{a}})Rupprecht, Laina, DiPietro, Baust,
  Tombari, Navab, and Hager]{rupprecht_learning_2016}
Rupprecht, C., Laina, I., DiPietro, R., Baust, M., Tombari, F., Navab, N., and
  Hager, G.~D.
\newblock Learning in an {Uncertain} {World}: {Representing} {Ambiguity}
  {Through} {Multiple} {Hypotheses}.
\newblock \emph{arXiv:1612.00197 [cs]}, December 2016{\natexlab{a}}.
\newblock URL \url{http://arxiv.org/abs/1612.00197}.
\newblock arXiv: 1612.00197.

\bibitem[Rupprecht et~al.(2016{\natexlab{b}})Rupprecht, Laina, DiPietro, Baust,
  Tombari, Navab, and Hager]{rupprecht_learning_2016-1}
Rupprecht, C., Laina, I., DiPietro, R., Baust, M., Tombari, F., Navab, N., and
  Hager, G.~D.
\newblock Learning in an {Uncertain} {World}: {Representing} {Ambiguity}
  {Through} {Multiple} {Hypotheses}.
\newblock \emph{arXiv:1612.00197 [cs]}, December 2016{\natexlab{b}}.
\newblock URL \url{http://arxiv.org/abs/1612.00197}.
\newblock arXiv: 1612.00197.

\bibitem[Salimans et~al.(2016)Salimans, Goodfellow, Zaremba, Cheung, Radford,
  and Chen]{salimans2016improved}
Salimans, T., Goodfellow, I., Zaremba, W., Cheung, V., Radford, A., and Chen,
  X.
\newblock Improved techniques for training gans.
\newblock In \emph{Advances in Neural Information Processing Systems}, pp.\
  2234--2242, 2016.

\bibitem[Schlegl et~al.(2017)Schlegl, Seeb{\"o}ck, Waldstein, Schmidt-Erfurth,
  and Langs]{schlegl2017unsupervised}
Schlegl, T., Seeb{\"o}ck, P., Waldstein, S.~M., Schmidt-Erfurth, U., and Langs,
  G.
\newblock Unsupervised anomaly detection with generative adversarial networks
  to guide marker discovery.
\newblock In \emph{International Conference on Information Processing in
  Medical Imaging}, pp.\  146--157. Springer, 2017.

\bibitem[Sch{\"o}lkopf et~al.(2001)Sch{\"o}lkopf, Platt, Shawe-Taylor, Smola,
  and Williamson]{scholkopf2001estimating}
Sch{\"o}lkopf, B., Platt, J.~C., Shawe-Taylor, J., Smola, A.~J., and
  Williamson, R.~C.
\newblock Estimating the support of a high-dimensional distribution.
\newblock \emph{Neural computation}, 13\penalty0 (7):\penalty0 1443--1471,
  2001.

\bibitem[Tax \& Duin(2004)Tax and Duin]{tax2004support}
Tax, D.~M. and Duin, R.~P.
\newblock Support vector data description.
\newblock \emph{Machine learning}, 54\penalty0 (1):\penalty0 45--66, 2004.

\bibitem[Zong et~al.(2018)Zong, Song, Min, Cheng, Lumezanu, Cho, and
  Chen]{zong2018deep}
Zong, B., Song, Q., Min, M.~R., Cheng, W., Lumezanu, C., Cho, D., and Chen, H.
\newblock Deep autoencoding gaussian mixture model for unsupervised anomaly
  detection.
\newblock \emph{International Conference on Learning Representations.}, 2018.

\end{thebibliography}


\begin{thebibliography}{47}
\providecommand{\natexlab}[1]{#1}
\providecommand{\url}[1]{\texttt{#1}}
\expandafter\ifx\csname urlstyle\endcsname\relax
  \providecommand{\doi}[1]{doi: #1}\else
  \providecommand{\doi}{doi: \begingroup \urlstyle{rm}\Url}\fi

\bibitem[Albertini \& de~Mello(2007)Albertini and de~Mello]{albertini2007self}
Marcelo~Keese Albertini and Rodrigo~Fernandes de~Mello.
\newblock A self-organizing neural network for detecting novelties.
\newblock In \emph{Proceedings of the 2007 ACM symposium on Applied computing},
  pp.\  462--466. ACM, 2007.

\bibitem[Arjovsky et~al.(2017)Arjovsky, Chintala, and
  Bottou]{arjovsky2017wasserstein}
Martin Arjovsky, Soumith Chintala, and L{\'e}on Bottou.
\newblock Wasserstein gan.
\newblock \emph{arXiv preprint arXiv:1701.07875}, 2017.

\bibitem[Arora \& Zhang(2017)Arora and Zhang]{arora2017gans}
Sanjeev Arora and Yi~Zhang.
\newblock Do gans actually learn the distribution? an empirical study.
\newblock \emph{arXiv preprint arXiv:1706.08224}, 2017.

\bibitem[Augusteijn \& Folkert(2002)Augusteijn and
  Folkert]{augusteijn2002neural}
MF~Augusteijn and BA~Folkert.
\newblock Neural network classification and novelty detection.
\newblock \emph{International Journal of Remote Sensing}, 23\penalty0
  (14):\penalty0 2891--2902, 2002.

\bibitem[Barreto \& Aguayo(2009)Barreto and Aguayo]{barreto2009time}
Guilherme~A Barreto and Leonardo Aguayo.
\newblock Time series clustering for anomaly detection using competitive neural
  networks.
\newblock In \emph{International Workshop on Self-Organizing Maps}, pp.\
  28--36. Springer, 2009.

\bibitem[Bengio \& LeCun(2007)Bengio and LeCun]{Bengio+chapter2007}
Yoshua Bengio and Yann LeCun.
\newblock Scaling learning algorithms towards {AI}.
\newblock In \emph{Large Scale Kernel Machines}. MIT Press, 2007.

\bibitem[Bishop(1994)]{bishop1994mixture}
Christopher~M Bishop.
\newblock Mixture density networks.
\newblock Technical report, Citeseer, 1994.

\bibitem[Chandola et~al.(2009)Chandola, Banerjee, and
  Kumar]{chandola2009anomaly}
Varun Chandola, Arindam Banerjee, and Vipin Kumar.
\newblock Anomaly detection: A survey.
\newblock \emph{ACM computing surveys (CSUR)}, 41\penalty0 (3):\penalty0 15,
  2009.

\bibitem[Clifton et~al.(2011)Clifton, Hugueny, and
  Tarassenko]{clifton2011novelty}
David~Andrew Clifton, Samuel Hugueny, and Lionel Tarassenko.
\newblock Novelty detection with multivariate extreme value statistics.
\newblock \emph{Journal of signal processing systems}, 65\penalty0
  (3):\penalty0 371--389, 2011.

\bibitem[Dai et~al.(2017)Dai, Yang, Yang, Cohen, and
  Salakhutdinov]{dai2017good}
Zihang Dai, Zhilin Yang, Fan Yang, William~W Cohen, and Ruslan~R Salakhutdinov.
\newblock Good semi-supervised learning that requires a bad gan.
\newblock In \emph{Advances in Neural Information Processing Systems}, pp.\
  6510--6520, 2017.

\bibitem[Deng \& Kasabov(2003)Deng and Kasabov]{deng2003line}
Da~Deng and Nikola Kasabov.
\newblock On-line pattern analysis by evolving self-organizing maps.
\newblock \emph{Neurocomputing}, 51:\penalty0 87--103, 2003.

\bibitem[Dey et~al.(2015)Dey, Ramakrishna, Hebert, and
  Andrew~Bagnell]{dey2015predicting}
Debadeepta Dey, Varun Ramakrishna, Martial Hebert, and J~Andrew~Bagnell.
\newblock Predicting multiple structured visual interpretations.
\newblock In \emph{Proceedings of the IEEE International Conference on Computer
  Vision}, pp.\  2947--2955, 2015.

\bibitem[Donahue et~al.(2016)Donahue, Kr{\"a}henb{\"u}hl, and
  Darrell]{donahue2016adversarial}
Jeff Donahue, Philipp Kr{\"a}henb{\"u}hl, and Trevor Darrell.
\newblock Adversarial feature learning.
\newblock \emph{arXiv preprint arXiv:1605.09782}, 2016.

\bibitem[Dosovitskiy \& Brox(2016)Dosovitskiy and
  Brox]{dosovitskiy2016generating}
Alexey Dosovitskiy and Thomas Brox.
\newblock Generating images with perceptual similarity metrics based on deep
  networks.
\newblock In \emph{Advances in Neural Information Processing Systems}, pp.\
  658--666, 2016.

\bibitem[Dumoulin et~al.(2017)Dumoulin, Belghazi, Poole, Mastropietro, Lamb,
  Arjovsky, and Courville]{dumoulin2017adversarially}
Vincent Dumoulin, Ishmael Belghazi, Ben Poole, Olivier Mastropietro, Alex Lamb,
  Martin Arjovsky, and Aaron Courville.
\newblock Adversarially learned inference.
\newblock \emph{International Conference on Learning Representations.}, 2017.

\bibitem[Gal \& Ghahramani(2016)Gal and Ghahramani]{gal2016dropout}
Yarin Gal and Zoubin Ghahramani.
\newblock Dropout as a bayesian approximation: Representing model uncertainty
  in deep learning.
\newblock In \emph{international conference on machine learning}, pp.\
  1050--1059, 2016.

\bibitem[Garc{\'\i}A-Rodr{\'\i}Guez et~al.(2012)Garc{\'\i}A-Rodr{\'\i}Guez,
  Angelopoulou, Garc{\'\i}A-Chamizo, Psarrou, Escolano, and
  Gim{\'e}Nez]{garcia2012autonomous}
Jos{\'e} Garc{\'\i}A-Rodr{\'\i}Guez, Anastassia Angelopoulou, Juan~Manuel
  Garc{\'\i}A-Chamizo, Alexandra Psarrou, Sergio~Orts Escolano, and
  Vicente~Morell Gim{\'e}Nez.
\newblock Autonomous growing neural gas for applications with time constraint:
  optimal parameter estimation.
\newblock \emph{Neural Networks}, 32:\penalty0 196--208, 2012.

\bibitem[Goodfellow et~al.(2014)Goodfellow, Pouget-Abadie, Mirza, Xu,
  Warde-Farley, Ozair, Courville, and Bengio]{goodfellow2014generative}
Ian Goodfellow, Jean Pouget-Abadie, Mehdi Mirza, Bing Xu, David Warde-Farley,
  Sherjil Ozair, Aaron Courville, and Yoshua Bengio.
\newblock Generative adversarial nets.
\newblock In \emph{Advances in neural information processing systems}, pp.\
  2672--2680, 2014.

\bibitem[Grubbs(1969)]{grubbs1969procedures}
Frank~E Grubbs.
\newblock Procedures for detecting outlying observations in samples.
\newblock \emph{Technometrics}, 11\penalty0 (1):\penalty0 1--21, 1969.

\bibitem[Gulrajani et~al.(2017)Gulrajani, Ahmed, Arjovsky, Dumoulin, and
  Courville]{gulrajani2017improved}
Ishaan Gulrajani, Faruk Ahmed, Martin Arjovsky, Vincent Dumoulin, and Aaron~C
  Courville.
\newblock Improved training of wasserstein gans.
\newblock In \emph{Advances in Neural Information Processing Systems}, pp.\
  5767--5777, 2017.

\bibitem[Guzman-Rivera et~al.(2014)Guzman-Rivera, Kohli, Batra, and
  Rutenbar]{pmlr-v33-guzman-rivera14}
Abner Guzman-Rivera, Pushmeet Kohli, Dhruv Batra, and Rob Rutenbar.
\newblock {Efficiently Enforcing Diversity in Multi-Output Structured
  Prediction}.
\newblock In Samuel Kaski and Jukka Corander (eds.), \emph{Proceedings of the
  Seventeenth International Conference on Artificial Intelligence and
  Statistics}, volume~33 of \emph{Proceedings of Machine Learning Research},
  pp.\  284--292, Reykjavik, Iceland, 22--25 Apr 2014. PMLR.
\newblock URL \url{http://proceedings.mlr.press/v33/guzman-rivera14.html}.

\bibitem[Hinton et~al.(2006)Hinton, Osindero, and Teh]{Hinton06}
Geoffrey~E. Hinton, Simon Osindero, and Yee~Whye Teh.
\newblock A fast learning algorithm for deep belief nets.
\newblock \emph{Neural Computation}, 18:\penalty0 1527--1554, 2006.

\bibitem[Hristozov et~al.(2007)Hristozov, Oprea, and
  Gasteiger]{hristozov2007ligand}
Dimitar Hristozov, Tudor~I Oprea, and Johann Gasteiger.
\newblock Ligand-based virtual screening by novelty detection with
  self-organizing maps.
\newblock \emph{Journal of chemical information and modeling}, 47\penalty0
  (6):\penalty0 2044--2062, 2007.

\bibitem[Ilg et~al.(2018)Ilg, Çiçek, Galesso, Klein, Makansi, Hutter, and
  Brox]{ilg_uncertainty_2018}
Eddy Ilg, Özgün Çiçek, Silvio Galesso, Aaron Klein, Osama Makansi, Frank
  Hutter, and Thomas Brox.
\newblock Uncertainty {Estimates} for {Optical} {Flow} with
  {Multi}-{Hypotheses} {Networks}.
\newblock \emph{arXiv:1802.07095 [cs]}, February 2018.
\newblock URL \url{http://arxiv.org/abs/1802.07095}.
\newblock arXiv: 1802.07095.

\bibitem[Kingma \& Welling(2013)Kingma and Welling]{kingma2013auto}
Diederik~P Kingma and Max Welling.
\newblock Auto-encoding variational bayes.
\newblock \emph{arXiv preprint arXiv:1312.6114}, 2013.

\bibitem[Kingma et~al.(2016)Kingma, Salimans, Jozefowicz, Chen, Sutskever, and
  Welling]{kingma2016improved}
Diederik~P Kingma, Tim Salimans, Rafal Jozefowicz, Xi~Chen, Ilya Sutskever, and
  Max Welling.
\newblock Improved variational inference with inverse autoregressive flow.
\newblock In \emph{Advances in Neural Information Processing Systems}, pp.\
  4743--4751, 2016.

\bibitem[Kliger \& Fleishman(2018)Kliger and Fleishman]{kliger2018novelty}
Mark Kliger and Shachar Fleishman.
\newblock Novelty detection with gan.
\newblock \emph{arXiv preprint arXiv:1802.10560}, 2018.

\bibitem[Lakshminarayanan et~al.(2017)Lakshminarayanan, Pritzel, and
  Blundell]{lakshminarayanan2017simple}
Balaji Lakshminarayanan, Alexander Pritzel, and Charles Blundell.
\newblock Simple and scalable predictive uncertainty estimation using deep
  ensembles.
\newblock In \emph{Advances in Neural Information Processing Systems}, pp.\
  6402--6413, 2017.

\bibitem[Larsen et~al.(2015)Larsen, S{\o}nderby, Larochelle, and
  Winther]{larsen2015autoencoding}
Anders Boesen~Lindbo Larsen, S{\o}ren~Kaae S{\o}nderby, Hugo Larochelle, and
  Ole Winther.
\newblock Autoencoding beyond pixels using a learned similarity metric.
\newblock \emph{arXiv preprint arXiv:1512.09300}, 2015.

\bibitem[Lee et~al.(2017)Lee, Hwang, Park, and Shin]{lee2017confident}
Kimin Lee, Changho Hwang, KyoungSoo Park, and Jinwoo Shin.
\newblock Confident multiple choice learning.
\newblock \emph{arXiv preprint arXiv:1706.03475}, 2017.

\bibitem[Lee et~al.(2016)Lee, Prakash, Cogswell, Ranjan, Crandall, and
  Batra]{lee2016stochastic}
Stefan Lee, Senthil Purushwalkam~Shiva Prakash, Michael Cogswell, Viresh
  Ranjan, David Crandall, and Dhruv Batra.
\newblock Stochastic multiple choice learning for training diverse deep
  ensembles.
\newblock In \emph{Advances in Neural Information Processing Systems}, pp.\
  2119--2127, 2016.

\bibitem[Liu et~al.(2008)Liu, Ting, and Zhou]{liu2008isolation}
Fei~Tony Liu, Kai~Ming Ting, and Zhi-Hua Zhou.
\newblock Isolation forest.
\newblock In \emph{2008 Eighth IEEE International Conference on Data Mining},
  pp.\  413--422. IEEE, 2008.

\bibitem[Liu et~al.(2012)Liu, Ting, and Zhou]{liu2012isolation}
Fei~Tony Liu, Kai~Ming Ting, and Zhi-Hua Zhou.
\newblock Isolation-based anomaly detection.
\newblock \emph{ACM Transactions on Knowledge Discovery from Data (TKDD)},
  6\penalty0 (1):\penalty0 3, 2012.

\bibitem[Makhzani et~al.(2015)Makhzani, Shlens, Jaitly, Goodfellow, and
  Frey]{makhzani2015adversarial}
Alireza Makhzani, Jonathon Shlens, Navdeep Jaitly, Ian Goodfellow, and Brendan
  Frey.
\newblock Adversarial autoencoders.
\newblock \emph{arXiv preprint arXiv:1511.05644}, 2015.

\bibitem[Marsland et~al.(2002)Marsland, Shapiro, and Nehmzow]{marsland2002self}
Stephen Marsland, Jonathan Shapiro, and Ulrich Nehmzow.
\newblock A self-organising network that grows when required.
\newblock \emph{Neural networks}, 15\penalty0 (8-9):\penalty0 1041--1058, 2002.

\bibitem[Parzen(1962)]{parzen1962estimation}
Emanuel Parzen.
\newblock On estimation of a probability density function and mode.
\newblock \emph{The annals of mathematical statistics}, 33\penalty0
  (3):\penalty0 1065--1076, 1962.

\bibitem[Patcha \& Park(2007)Patcha and Park]{patcha2007overview}
Animesh Patcha and Jung-Min Park.
\newblock An overview of anomaly detection techniques: Existing solutions and
  latest technological trends.
\newblock \emph{Computer networks}, 51\penalty0 (12):\penalty0 3448--3470,
  2007.

\bibitem[Pimentel et~al.(2014)Pimentel, Clifton, Clifton, and
  Tarassenko]{pimentel2014review}
Marco~AF Pimentel, David~A Clifton, Lei Clifton, and Lionel Tarassenko.
\newblock A review of novelty detection.
\newblock \emph{Signal Processing}, 99:\penalty0 215--249, 2014.

\bibitem[Radford et~al.(2015)Radford, Metz, and
  Chintala]{DBLP:journals/corr/RadfordMC15}
Alec Radford, Luke Metz, and Soumith Chintala.
\newblock Unsupervised representation learning with deep convolutional
  generative adversarial networks.
\newblock \emph{CoRR}, abs/1511.06434, 2015.
\newblock URL \url{http://arxiv.org/abs/1511.06434}.

\bibitem[Rezende et~al.(2014)Rezende, Mohamed, and
  Wierstra]{rezende2014stochastic}
Danilo~Jimenez Rezende, Shakir Mohamed, and Daan Wierstra.
\newblock Stochastic backpropagation and approximate inference in deep
  generative models.
\newblock \emph{arXiv preprint arXiv:1401.4082}, 2014.

\bibitem[Rupprecht et~al.(2016)Rupprecht, Laina, DiPietro, Baust, Tombari,
  Navab, and Hager]{rupprecht_learning_2016-1}
Christian Rupprecht, Iro Laina, Robert DiPietro, Maximilian Baust, Federico
  Tombari, Nassir Navab, and Gregory~D. Hager.
\newblock Learning in an {Uncertain} {World}: {Representing} {Ambiguity}
  {Through} {Multiple} {Hypotheses}.
\newblock \emph{arXiv:1612.00197 [cs]}, December 2016.
\newblock URL \url{http://arxiv.org/abs/1612.00197}.
\newblock arXiv: 1612.00197.

\bibitem[Schlegl et~al.(2017)Schlegl, Seeb{\"o}ck, Waldstein, Schmidt-Erfurth,
  and Langs]{schlegl2017unsupervised}
Thomas Schlegl, Philipp Seeb{\"o}ck, Sebastian~M Waldstein, Ursula
  Schmidt-Erfurth, and Georg Langs.
\newblock Unsupervised anomaly detection with generative adversarial networks
  to guide marker discovery.
\newblock In \emph{International Conference on Information Processing in
  Medical Imaging}, pp.\  146--157. Springer, 2017.

\bibitem[Sch{\"o}lkopf et~al.(2001)Sch{\"o}lkopf, Platt, Shawe-Taylor, Smola,
  and Williamson]{scholkopf2001estimating}
Bernhard Sch{\"o}lkopf, John~C Platt, John Shawe-Taylor, Alex~J Smola, and
  Robert~C Williamson.
\newblock Estimating the support of a high-dimensional distribution.
\newblock \emph{Neural computation}, 13\penalty0 (7):\penalty0 1443--1471,
  2001.

\bibitem[Singh \& Markou(2004)Singh and Markou]{singh2004approach}
Sameer Singh and Markos Markou.
\newblock An approach to novelty detection applied to the classification of
  image regions.
\newblock \emph{IEEE Transactions on Knowledge and Data Engineering},
  16\penalty0 (4):\penalty0 396--407, 2004.

\bibitem[Tax \& Duin(2004)Tax and Duin]{tax2004support}
David~MJ Tax and Robert~PW Duin.
\newblock Support vector data description.
\newblock \emph{Machine learning}, 54\penalty0 (1):\penalty0 45--66, 2004.

\bibitem[van~den Oord et~al.(2017)van~den Oord, Vinyals, et~al.]{van2017neural}
Aaron van~den Oord, Oriol Vinyals, et~al.
\newblock Neural discrete representation learning.
\newblock In \emph{Advances in Neural Information Processing Systems}, pp.\
  6306--6315, 2017.

\bibitem[Zong et~al.(2018)Zong, Song, Min, Cheng, Lumezanu, Cho, and
  Chen]{zong2018deep}
Bo~Zong, Qi~Song, Martin~Renqiang Min, Wei Cheng, Cristian Lumezanu, Daeki Cho,
  and Haifeng Chen.
\newblock Deep autoencoding gaussian mixture model for unsupervised anomaly
  detection.
\newblock \emph{International Conference on Learning Representations.}, 2018.

\end{thebibliography}
\bibliographystyle{icml2019}

%
%
%
%
%

 \appendix
Definitions and more formal discussions are provided in these sections.

\section{Mixture Density Network}
\label{Appendix:MDN}
The Mixture Density networks predict a data conditional  Gaussian mixture model (GMM)in the data space. Conditioning means that each latent vector, i.e., a point on the learned manifold is projected back to a GMM in the data space.

A GMM learns from the following energy function:
\begin{equation}
L_{GMM}(x) = -\log \sum_h \alpha_h \mathcal{N}(x;\mu_{h},\sigma_h)
\end{equation}
Whereby $x$ is the input data, $\mu_h$ and $\sigma_h$ parametrize the $h-th$ Gaussian distribution  in the mixture. $\alpha_h$ are the mixing coefficients across the individual mixtures.

Contrary, a Mixture Density network hat multiple output heads (multiple-hypotheses). The framework extends the GMM-learning by the data conditioning as follows:
\begin{equation}
L_{MDN}(x) =  E_{z_i \sim q_{\phi}(z_i|x)}
\left[ L_{GMM}(x|z_i)\right]  
\label{eq:-MDN}
\end{equation}
whereby $q_\phi$ is a inference network shared by all individual mixtures. $z$ is the latent code.  
The hypotheses are coupled into forming a likelihood function by the mixing coefficients $\alpha_i$.

\section{Multimodal learning on the flipped moon toy dataset}
\label{app:flipped-moon-dataest}

\begin{figure}[h]
	\begin{center}
		\includegraphics[width=1\linewidth]{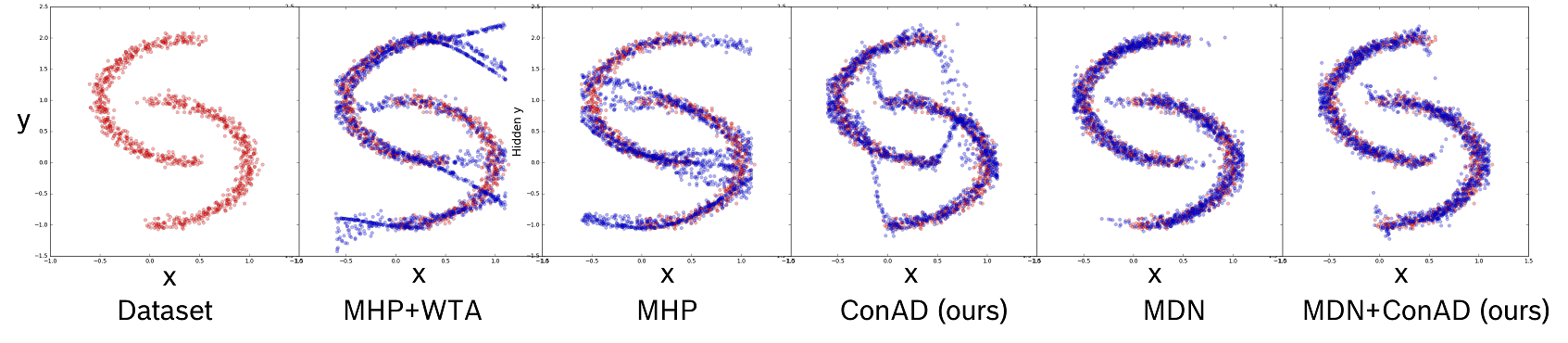}
	\end{center}
	\caption{Flipped half-moon dataset: conditional prediction of $y$ based on $x$. Red points are samples from true distribution while blue points represent samples from distributions approximations. Learning with multiple-hypotheses predictions (MHP) loss or MHP + Winner-takes-all (WTA) loss lead to support of artificial data regions. Mixture density networks and our approach ConAD reduces this effect. }

	\label{fig:flipped half moon}
\end{figure}

Fig. \ref{fig:flipped half moon} shows the flipped half-moon dataset to demonstrate MHP-learning in contrast to unimodal output distribution learning. In this section, Fig \ref{fig:flipped half moon} shows a qualitative evaluation of different MHP-techniques. This task is a one-to-many mapping from $x$ to $y$ with a discontinuity at the point $x=0$ and $x=0.5$.

When the local density function abruptly ends, MHP-techniques support artificial data regions since they are not penalized for artificial modes by the objective function as discussed before. We refer to this property as an inconsistency concerning the true underlying distribution. In contrast to that, Mixture Density Networks (MDN) and our ConADs approaches reduce the inconsistencies to the minimum.

\section{One-to-many mapping tasks require multi-modality}

Consider a simple toy problem with an observable $x$ and hidden $y$ which is to be predicted and expressed by the conditional distribution $p_{true}(y|x)$ such as in Fig. \ref{fig:flipped half moon}. Since the data conditional is multi-modal for some $x$, an uni-modal output distribution cannot fully capture the underlying distribution. Instead, the bias-free solution for the Mean-Squared-Error-minimizer  is the empirical mean  $\overline{y_{x_i}}$ of  $p_{train}(y|x_i)$ on the training set. However, this learned conditional density does not comply with the underlying distribution: sampled data points fall into the low-likelihood regions under $p_{true}(y|x)$.
With increasing number of output hypotheses, the data modes could be gradually captured. For  this task, the energy to be minimized  is given by the Negative-log-likelihood of the Mixture Density Network (MDN) App. \ref{Appendix:MDN} under a Gaussian Mixture with hypotheses h in  Eq. \ref{Eq:MDN-toy task} : 
\begin{dmath}
	E_{MDN}(\Theta) = - \log L(\Theta|X;Y) = -\log p_{GMM}(Y|X,\Theta ) = - \sum_{i} \sum_{h} \log \alpha_h p_{\theta_h}(y_i|x_i)
	\label{Eq:MDN-toy task-1}
\end{dmath}
with 
\begin{equation}
p_{\theta_h}(y_i|x_i,\theta_h) = \frac{1}{\sqrt{2\pi}\sigma_h} \exp-\frac{(y_i-\mu_h)^2}{2\sigma_h^2}
\label{Eq:MDN-toy task}
\end{equation}

\section{Lemma 4.1}
\label{Appendix:Lemma 4.1}
Given a sufficient number of hypotheses H', an optimal solution $\Theta^*$ for $E_{WTA}(\Theta^{*})$ is not unique (permutation is excluded). There exists a $\Theta^{'}$ with $E_{WTA}(\Theta^{*}) = E_{WTA}(\Theta^{'})$ which is not consistent w.r.t. the underlying output distribution $p_{train}(y_i|x_i)$.

\begin{proof}: Suppose $c$ is the maximal modes count of the dataset sampled from the real underlying conditional output distribution $p(y_i|x_i)$. Since $\left|\left\lbrace (x_i,y_i)\right\rbrace\right | < \infty \rightarrow  c < \infty$. 
	
	Suppose $H=c$, then a trivial optimal solution for $E_{WTA}(\Theta_H)$ is found by centering each hypothesis $\mu_{ik}$ at a different empirical data point $k$ $y_{ik} \sim (y_i,x_i)$ and $\sigma_{ik} \mapsto 0$. In this case $\lim\limits_{\sigma_{ik}\mapsto 0;\forall i,k} E_{WTA}(\widehat{\Theta}_H) = 0$. 
	
	Suppose $H'>c$, then a solution $\widehat{\Theta}_{H'}$ can be formulated s.t.:   $E(\widehat{\Theta}_H) = E(\widehat{\Theta}_{H'})$.  
	
	Let $\widehat{\Theta}_{H'}= \widehat{\Theta}_H \cup \widehat{\Theta}_{H+1\dots H'}= \widehat{\Theta}_H \cup \left\lbrace \theta_{h+1}\dots \theta_{h'}\right\rbrace$ for some \textbf{ random $\widehat{\Theta}_{H+1\dots H'}$}. Due to randomness and without loss of generality, one can assume that $\forall (x_i,y_i), \forall\theta_i \in \Theta_{H+1\dots H'}$, $\theta_i$ is not the optimal hypothesis for any training point $(x_i, y_i) \in D_{train}$.

	In this case due to the winner-takes-all energy formulation we have: 
	
	\begin{dmath}
		E_{WTA}(\widehat{\theta}_{H'}) =  -\sum_{i} \max_{1\leq h\leq H'} \log  p_{\theta_h}(y_i|x_i)  = -\sum_{i} \max_{1\leq h\leq H} \log  p_{\theta_h}(y_i|x_i) 
		=E_{WTA}(\widehat{\theta}_{H})\
	\end{dmath}
	So $\widehat{\Theta}_{H}$ and $\widehat{\Theta}_{H'}$ with $H'>H$ are both solutions to the loss formulation and share the same energy level. The extended hypotheses can support arbitrary artificial data regions without being penalized.
\end{proof}
\section{Lemma 4.2}

\begin{dmath}
	E_{MHP}(\Theta) =  - \sum_{i} \sum_{h} \log \left( p_{\theta_h}(y_i|x_i) \right) *  \begin{cases}
		1-\epsilon, p_{\theta_h}(y_i|x_i) \geq p_{\theta_k}(y_i|x_i), \forall k \\
		\frac{\epsilon}{H-1},\text{else}
	\end{cases}
	\label{Eq:MHP-Ruprechttoy task}
\end{dmath}

Whereby $x_i$,$y_i$ is corresponding input-output pairs from the training dataset, $1\leq h\leq H$ is a hypothesis branch, which is generated by a parametrized neural network with the parameter set $\theta_h$. Furthermore,  $\epsilon$  is a hyperparameter used to distribute the learning signal to the non-optimal hypotheses. $\Theta$ is the collection of all $\theta_h$.
\label{Appendix:Lemma 4.2}
\begin{lemma}
	Similar to Lemma \ref{Appendix:Lemma 4.1}, minimizing $E_{MHP}$ in Eq. \ref{Eq:MHP-Ruprechttoy task} might also lead to an inconsistent approximation of the real underlying output distribution.
\end{lemma}    

\begin{proof} First, note that $0 \leq \epsilon\leq \frac{H-1}{H}$, since $\epsilon <  0$ would push away non-locally optimal hypotheses from the empirical solution, $\epsilon > \frac{H-1}{H}$  would penalize the best hypothesis more than others. Both are undesired properties of MHP-learning. First consider the case where $\epsilon \mapsto  \frac{H-1}{H}$ :
	
	\begin{equation}
	\lim\limits_{\epsilon\mapsto \frac{H-1}{H}} E_{MHP}(\Theta)=  \sum_{i} \sum_{h} \log \left( p_{\theta_h}(y_i|x_i) \right) * \frac{1}{H}
	\end{equation}
	\begin{align*}
	&= \frac{1}{H}\sum_{h} \left(\sum_{i} \log \left( p_{\theta_h}(y_i|x_i) \right)\right) \\
	&= \frac{1}{H}\sum_{h} E_{\theta_{h}}\\
	\end{align*}
	$\forall \theta_h$ and training data points $(x_i,y_{ik})$ the optimal least-squares solution is the mean, therefore we have:
	\begin{align*}
	\theta_h^{*}(y_i|x_i)&=  E_{y_{ik \sim p(y|x_i)}[y_i]} \\
	&=\frac{1}{l} \sum_{i=1}^{l} y_i ; y_{ik} \sim p(y_i|x_i)\\
	\end{align*}
	In this case, all hypotheses are optimized independently and converge to the same solution similar to a single-hypothesis approach. The resulting distribution is inconsistent w.r.t the real output distribution (see Fig. \ref{fig:flipped half moon} for an example).  
	
	Now consider $\epsilon \mapsto 1$:
	\begin{dmath}
		\lim\limits_{\epsilon\mapsto 1}E_{MHP}(\Theta)= - \sum_{i} \sum_{h} \log \left( p_{\theta_h}(y_i|x_i) \right) *  \begin{cases}
			1;  \text{if } \theta_h \text{ is best hypothesis} \\
			0;\text{else}
		\end{cases}\\
		= -\sum_{i} \max_{1\leq h\leq H'} \log  p_{\theta_h}(y_i|x_i)\\
		= E_{WTA}(\Theta) 
	\end{dmath}
	In this case $E_{MHP}$ shares the same inconsistency property with $E_{WTA}$. Consequently, choosing  $\epsilon \in [0,\frac{H-1}{H}]$ only smoothes the penalty on suboptimal hypotheses. The risk remains that distributions induced by non-optimal hypotheses are beyond the real modes of the underlying distribution. 
	
\end{proof}
\section{Related works in detail}
Traditional one-class learning techniques \citep{scholkopf2001estimating,tax2004support,liu2008isolation,liu2012isolation,breunig2000lof} often fail in high-dimensional input domains and require careful feature selection \citep{zong2018deep}.  

To cope with high-dimensional domains, typically a reconstruction-based approach is used.  This paradigm learns the normal data distribution during training and uses the data likelihood as an anomaly score at test time. 
Recently, advances in generative modeling such as Generative Adversarial Network (GAN) \citep{goodfellow2014generative} and Variational Autoencoder (VAE) \citep{rezende2014stochastic,kingma2013auto} are used for anomaly detection  \citep{zong2018deep, schlegl2017unsupervised, deecke2018anomaly}.  However, GAN and VAE approaches have limitations in anomaly detection tasks. The GAN tends to assign less probability mass to real samples, while VAE typically regresses to the conditional means. The mean regression in VAE express the model uncertainty and falsify the reconstruction-errors for unseen images.

To address model uncertainty in VAE, the decoder is given additional expressive power with multi-headed decoders. The idea is to approximate multiple conditional modes (dense data regions) by using networks with multiple heads. This leads to training of multiple networks in Multi-Choice-learning \citep{dey2015predicting,lee2017confident,lee2016stochastic}, the estimation of a conditional Gaussian Mixture model in Mixture Density Networks (MDN) \citep{bishop1994mixture}, and multiple-hypotheses predictions (MHP) \citep{koltun_multi_choise,bhattacharyya2018accurate,rupprecht_learning_2016,ilg_uncertainty_2018}. In MDN, the mixtures are strictly coupled via mixture coefficients while mixtures in MHPs act as loosely coupled local density estimators. In MHP, only the best hypothesis branch will receive a learning signal, i.e., the one that best explains the training sample.

For anomaly detection, our model uses MHP-training with a VAE to address the model uncertainty directly. In MDN, the anomaly score is proportional to the weighted distances to all data modes, and in MHP only to closest data mode. To highlight the change in paradigm, we refer to this learning in MHP as consistency-based learning. Samples have a small effect on the loss as long as they are close to one single data mode. The learning dynamic in MHP is also different and more efficient than in MDN: the number of samples with a large loss is much lower. In this sense, we relax the learning objective from strict density-based to consistency-based learning.

This is related to the Local Outlier Factor (LOF) approach \citep{breunig2000lof}, where the outlier-score only depends on the local neighborhood. 
In LOF, the outlier score is proportional to the mean density of neighboring points divided by the local point density. Hence, distant samples do not influence the outlier-score. Motivated by this heuristic, our model employs learning of many loosely decoupled local density estimates with MHP-learning. %
While LOF computes the outlier score only at test time and directly in the input space, our model first approximates the data manifold and subsequently performs anomaly detection in the input space under the learned model.

The MHP-technique has been used for uncertainty estimation in  tasks like future prediction \citep{rupprecht_learning_2016-1} or optical flow prediction \citep{ilg_uncertainty_2018}.   In the simplest form, the multiple network heads learn from a winner-takes-all (WTA) loss, whereby only the best branch receives the learning signal. These works extended the loss with local smoothness terms \citep{ilg_uncertainty_2018} or  distribution of the learning signal also to the other, non-optimal branches \citep{rupprecht_learning_2016-1} to generate diverse and meaningful hypotheses.   

The major problem of MHP-approaches is that areas not supported by samples can be covered by unused hypotheses. This is fatal for anomaly detection. Therefore, our ConAD approach employs a discriminator D to assess the quality of the generated hypotheses to avoid support of non-existent data modes. To avoid mode collapse due to the GAN framework, we employ hypotheses discrimination. In the spirit of minibatch discrimination \citep{salimans2016improved}, D additionally receives pair-wise distances across a batch of hypotheses. Since a batch of real samples is typically diverse, D can detect a homogeneous batch of hypotheses as fake easily.

\section{Detailed performance on CIFAR-10}
\begin{table*}[h]
	\centering
	\caption{CIFAR-10 anomaly detection: AUROC-performance of different approaches. The column indicates which class was used as in-class data for distribution learning. Note that random performance is at 50\% and higher scores are better. Top-2-methods are marked. Our ConAD approach outperforms traditional methods and vanilla MHP-approaches significantly and can benefit from an increasing number of hypotheses.
	}
	\vskip .15in
	\begin{small}
		\begin{sc}
			\begin{tabular}{l||ccccccccccc||c}%
				CIFAR-10 & 0 & 1 & 2 & 3 & 4 & 5 & 6 & 7 & 8 & 9 & Mean\\
				\hline
				KDE-PCA           & 70.5 & 49.3 &\textbf{73.4} & 52.2 & 69.1 & 43.9 &   {77.1} & 45.8 & 59.5 & 49.0&59.0 \\
				KDE-Alexnet       & 55.9 & 48.7 &58.2 & 53.1 & 65.1 & 55.1 &   61.3 & 59.3 & 60.0 & 52.9&57.0 \\
				OC-SVM-PCA        & 66.6 & 47.3 &67.5 & 53.0 & 82.7 & 43.8 &   \textbf{78.7} & 53.2 & 72.0 & 45.3&61.0 \\
				OC-SVM-Alexnet    & 59.4 & 54.0 &58.8 & 57.5 & 75.3 & 55.8 &   69.2 & 54.7 & 63.0 & 53.0&60.1 \\
				IF                & 63.0 & 37.9 &63.0 & 40.8 & \textbf{76.4} & 51.4 &   66.6 & 48.0 & 65.1 & 45.9&55.8 \\
				GMM               & 70.9 & 44.3 &69.7 & 44.5 & \textbf{76.1} & 50.5 &   76.6 & 49.6 & 64.6 & 38.4&58.5 \\
				\hline
				AnoGAN       & 61.0  & 56.5  & 64.8  & 52.8  & 67.0  & 59.2  & 62.5  & 57.6  & 72.3  & 58.2  & 61.2  \\
				ADGAN & 63.2 & 52.9 & 58.0 & \textbf{60.6} & 60.7 & \textbf{65.9} & 61.1 & 63.0 & 74.4 & 64.4 & 62. \\

				VAE   & 77.1  & 46.7   & 68.4   & 53.8   & 71.    & 54.2   & 64.2   & 51.2   & \textbf{76.5}   & 46.7   & 61.0\\ 
				VAEGAN & 76.2 &    46.9 &    \textbf{69.7} &    52.0 &     {75.6} &    53.6 &    58.8 &    55.4 &    75.4 &    46.0 &     60.9 \\
				OC-D-SVDD & 61.7& 65.9& 50.8& 59.1& 60.9& 65.7& 67.7& 67.3& 75.9& 73.1 & 63.2\\
				
				\hline
				
				MDN-2 & 76.1 &  46.9 &  68.7 &  53.8 &  70.4 &  53.8 &  63.2 &  52.3 &  \textbf{76.8} & 46.7  & 60.9 \\
				MDN-4 & 76.9 &  46.8 &  68.6 &  53.5 &  69.3 &  54.4 &  63.5 &  54.1 &  76.  & 46.9  & 61.0 \\
				MDN-8 & 76.2 &  46.9 &  68.6 &  53.3 &  70.4 &  54.7 &  63.3 &  53.  &  76.3 & 47.3  & 61. \\
				MDN-16 & 76.2 &  47.9 &  68.2 &  52.8 &  70.1 &  54.  &  63.5 &  52.9 &  76.4 & 46.9  & 60.9 \\
				\hline
				MHP-WTA-2 & 77.3 &  51.6 &  68.  &  55.2 &  69.5 &  54.3 &  64.3 &  55.5 &  76.  & 51.2  & 62.2 \\
				MHP-WTA-4 & \textbf{77.8} &  53.9 &  65.1 &  56.7 &  66.  &  54.2 &  63.5 &  56.3 &  75.2 & 54.1  & 62.2 \\
				MHP-WTA-8 & 76.1 &  56.  &  62.7 &  58.8 &  62.6 &  55.3 &  61.4 &  57.8 &  74.3 &  54.8 & 61.9 \\
				MHP-WTA-16 & 75.7 &  56.7 &  60.9 &  59.8 &  62.7 &  56.  &  61.  &  56.8 &  73.8 & 57.3  & 62. \\
				\hline
				MHP-2 & 75.5 &  49.9 &  67.6 &  54.6 &  69.3 &  54.3 &  63.6 &  57.7 &  76.4 & 50.8  & 61.9 \\
				MHP-4 & 75.2 &  51.  &  66.  &  56.8 &  67.7 &  55.1 &  64.4 &  56.  &  76.4 & 51.   & 61.9 \\
				MHP-8 & 75.7 &  54.  &  65.2 &  57.6 &  64.8 &  55.4 &  62.5 &  54.7 &  75.9 & 53.   & 61.8 \\
				MHP-16 & 75.8 &  53.9 &  64.1 &  58.5 &  64.6 &  55.2 &  62.3 &  54.5 &  75.9 & 53.2  & 61.7 \\
				\hline
				MDN+GAN-2 & 74.6 &  48.9 &  68.6 &  52.1 &  71.1 &  52.5 &  66.8 &  57.7 &  76.5 & 48.1  & 61.6 \\
				MDN+GAN-4 & 76.2 &  50.4 &  69.  &  52.4 &  71.6 &  53.2 &  65.9 &  58.3 &  75.3 & 48.9  & 62.1 \\
				MDN+GAN-8 & 77.4 &  48.3 &  69.3 &  53.1 &  72.2 &  53.7 &  67.9 &  54.  &  76.  & 51.9  & 62.3 \\
				MDN+GAN-16 & 73.6 &  46.9 &  69.4 &  52.2 &  75.3 &  54.1 &  65.7 &  56.8 &  75.3 & 45.4  & 61.4 \\
				\hline

				\hline
				
				ConAD - 2 (ours) & 77.3 &    60.0 &  66.6 &    56.2 &    69.4 &    56.1 &    70.6 &    63.0 &     74.8 &    49.9 &    64.3 \\
				ConAD - 4 (ours) & \textbf{77.6} &    52.5 &    66.3 &    57.0 &     68.7 &    54.1 &    \textbf{80.1} &    54.8 &    74.1 &    53.9 &    {63.9} \\
				ConAD - 8 (ours) & 77.4 &    \textbf{65.2} &    64.8 &    60.1 &    67.0 &     {57.9} &     72.5 &    \textbf{66.2} &    74.8 &    \textbf{66.0} &     \textbf{67.1} \\
				ConAD - 16 (ours)& 77.2& \textbf{63.1}& 63.1& \textbf{61.5}& 63.3& \textbf{58.8}& 69.1& \textbf{64.0}& 75.5& \textbf{63.7}& \textbf{65.9}\\
				
			\end{tabular} 
		\end{sc}
	\end{small}

	\label{Tab:CIFAR-10-all}
\end{table*}
\section{Metal anomaly results}

\begin{table}[h]
	\vskip 0.15in
	
	\centering
	\caption{Anomaly detection performance on Metal Anomaly dataset. Here the anomaly detection is measured by summing up reconstructions errors over all pixel positions. This consideration is rather sensitive to noise in very high-dimensional input space such as in Metal Anomaly. The best two models are marked.}
	\vskip 0.15in
	
	\begin{small}
		\begin{sc} 
			\begin{tabular}{ccccc}
				\toprule
				\midrule
				& \multicolumn{4}{c}{Hypotheses}\\
				\cmidrule{2-5}
				Model & 1 & 2 & 4 & 8  \\ 
				\midrule
				
				MHP & \multirow{3}{*}{79.5=VAE}& \textbf{87.6}  & \textbf{83.4}  & 79.3  \\
				MHP+WTA &  & 85.1  & \textbf{87.8} & 80.0  \\
				MDN & & 74.6  & 76.5  & 74.3  \\
				\cmidrule{2-5}
				MDN+GAN & \multirow{2}{*}{78.2 =VAEGAN} & 81.0 & 78.1  &  \textbf{81.0}  \\
				ConAD & & \textbf{86.7}& 81.2&  \textbf{81.7}\\
				
				\midrule
				\bottomrule
			\end{tabular}
		\end{sc}
	\end{small}
	\label{tab:100 percent}
\end{table}

\begin{table}[h]
	\vskip 0.15in
	\centering
	\caption{Anomaly detection performance on Metal Anomaly dataset by summing over the 1\% most-anomalous pixels for each input image. The  best two models are marked.}
	\vskip 0.15in
	
	\begin{small}
		\begin{sc} 
			\begin{tabular}{ccccc}
				\toprule
				\midrule
				& \multicolumn{4}{c}{Hypotheses}\\
				\cmidrule{2-5}
				Model & 1 & 2 & 4 & 8  \\ 
				\midrule
				MHP & \multirow{3}{*}{97.7=VAE}& \textbf{99.3}  & \textbf{99.0}  & \textbf{98.4}  \\
				MHP+WTA &  & 99.0  & \textbf{99.0} & 98.1  \\
				MDN & & 97.0  & 96.0  & 97.5  \\
				\cmidrule{2-5}
				MDN+GAN & \multirow{2}{*}{97.8 =VAEGAN} & 96.6 & 95.1  &  97.8  \\
				ConAD & & \textbf{99.2}& \textbf{99.0}&  \textbf{98.7}\\
				\midrule
				\bottomrule
			\end{tabular}
		\end{sc}
	\end{small}
	\label{tab:1percent}
	
\end{table}            

\end{document}